\def\year{2022}\relax
\documentclass[letterpaper]{article} 
\usepackage{aaai22}  
\usepackage{times}  
\usepackage{helvet}  
\usepackage{courier}  
\usepackage[hyphens]{url}  
\usepackage{graphicx} 
\usepackage{natbib}  
\usepackage{caption} 
\DeclareCaptionStyle{ruled}{labelfont=normalfont,labelsep=colon,strut=off} 
\usepackage{microtype}
\usepackage{graphicx}
\usepackage{subcaption}
\usepackage{booktabs} 

\newcommand{\mycc}{\cellcolor{lightgray}}

\usepackage{csquotes}
\usepackage{kotex, amsmath, amssymb, amsthm, mathtools, enumitem, systeme, bm, bbm, physics}
\usepackage[ruled,vlined,linesnumbered]{algorithm2e}
\usepackage{optidef}
\usepackage{arydshln}
\usepackage[table,svgnames]{xcolor}
\usepackage{multicol}
\usepackage{float}
\usepackage{inputenc}
\usepackage{thmtools}
\usepackage{thm-restate}

\newlist{enumthm}{enumerate}{1}
\setlist[enumthm]{label=(\Alph*)}

\newtheorem{definition}{Definition}
\makeatletter
\newcommand{\qualification}[1]{\ifthmt@thisistheone#1\fi}
\makeatother

\newtheorem{theorem}{Theorem}
\newtheorem{lemma}{Lemma}
\newtheorem{proposition}{Proposition}
\newtheorem{corollary}{Corollary}[theorem]

\newtheorem{remark}{Remark}
\newtheorem{assumption}{Assumption}

\DeclareMathOperator*{\argmin}{\arg\!\min}

\newcommand{\MMD}{\operatornamewithlimits{MMD}}

\SetKwInput{KwInput}{Input}
\SetKwInput{KwOutput}{Output}

\usepackage{hyperref}

\usepackage{newfloat}
\usepackage{listings}
\floatstyle{ruled}
\newfloat{listing}{tb}{lst}{}
\floatname{listing}{Listing}
\title{Fast and Efficient MMD-based Fair PCA via Optimization over Stiefel Manifold}
\author {
    Junghyun Lee,\textsuperscript{\rm 1}
    Gwangsu Kim*,\textsuperscript{\rm 2}
    Matt Olfat,\textsuperscript{\rm 3,4}
    Mark Hasegawa-Johnson,\textsuperscript{\rm 5}
    Chang D. Yoo*\textsuperscript{\rm 2}
}
\affiliations {
    \textsuperscript{\rm 1} Kim Jaechul Graduate School of AI, KAIST, Seoul, Republic of Korea\\
    \textsuperscript{\rm 2} School of Electrical Engineering, KAIST, Daejeon, Republic of Korea\\
    \textsuperscript{\rm 3} UC Berkeley IEOR, Berkely, CA, USA\\
    \textsuperscript{\rm 4} Citadel, Chicago, IL, USA\\
    \textsuperscript{\rm 5} Department of Electrical and Computer Engineering, University of Illinois Urbana-Champaign, IL, USA\\
    \{jh\_lee00, s88012, cd\_yoo\}@kaist.ac.kr, molfat@berkeley.edu, jhasegaw@illinois.edu
}

\begin{document}

\maketitle

\begin{abstract}
This paper defines fair principal component analysis (PCA) as minimizing the maximum mean discrepancy (MMD) between dimensionality-reduced conditional distributions of different protected classes.
The incorporation of MMD naturally leads to an exact and tractable mathematical formulation of fairness with good statistical properties.
We formulate the problem of fair PCA subject to MMD constraints as a non-convex optimization over the Stiefel manifold and solve it using the Riemannian Exact Penalty Method with Smoothing (REPMS; \citeauthor{repms}, \citeyear{repms}).
Importantly, we provide local optimality guarantees and explicitly show the theoretical effect of each hyperparameter in practical settings, extending previous results.
Experimental comparisons based on synthetic and UCI datasets show that our approach outperforms prior work in explained variance, fairness, and runtime.
\end{abstract}

\section{Introduction}
\label{sec:intro}
It has become increasingly evident that many widely-deployed machine learning algorithms are biased, yielding outcomes that can be discriminatory across key groupings such as race, gender and ethnicity \cite{MehrabiMSLG19}.
As the applications of these algorithms proliferate in protected areas like healthcare \cite{karan2012diagnosing}, hiring \cite{chien2008data} and criminal justice \cite{angwin2016machine}, this creates the potential for further exacerbating social biases.
To address this, there has recently been a surge of interest in ensuring fairness in resulting machine learning algorithms.

Working in high-dimensional spaces can be undesirable as 
the curse of dimensionality manifests in the form of data sparsity and computational intractability.
Various dimensionality reduction algorithms are deployed to resolve these issues, and Principal Component Analysis (PCA) \cite{Pearson1901, Hotelling1933}, is one of the most popular methods \cite{JolliffeC16}.
One particular advantage of PCA is that there's no need to train a complex neural network.

In this work, fair PCA is defined as performing dimensionality reduction while minimizing the difference in the conditional distributions of projections of different protected groups.
Here, the projected data can be considered as a dimension-reduced fair representation of the input data \cite{ZemelWSPD13}.
We answer the questions of 1) how fairness should be defined for PCA and 2) how to (algorithmically) incorporate fairness into PCA in a {\it fast} and {\it efficient} manner.
This work takes a different approach from prior studies on PCA fairness ~\cite{SamadiTMSV18, OA19}, which is discussed in Section \ref{sec:fair-dim} and \ref{sec:related-works}.


Our main contributions are as follows:

\begin{itemize}
    \item We motivate a new mathematical definition of fairness for PCA using the maximum-mean discrepancy (MMD), which can be evaluated in a computationally efficient manner from the samples while guaranteeing asymptotic consistency.
    Such properties were not available in the previous definition of fair PCA \cite{OA19}.
    This is discussed in detail in Section \ref{sec:fair-dim} and \ref{sec:statistical}.
    
    \item We formulate the task of performing MMD-based fair PCA as a constrained optimization over the Stiefel manifold and propose using REPMS \cite{repms}.
    For the first time, we prove two general theoretical guarantees of REPMS regarding the local minimality and feasibility.
    This is discussed in detail in Section \ref{sec:mbf-pca} and \ref{sec:repms}.
    
    \item Using synthetic and UCI datasets, we verify the efficacy of our approach in terms of explained variance, fairness, and runtime.
    Furthermore, we verify that using fair PCA does indeed result in a fair representation, as in \cite{ZemelWSPD13}.
    This is discussed in detail in Section \ref{sec:experiments}.
\end{itemize}

\begin{figure*}[!t]
	\begin{center}
		\begin{subfigure}[t]{0.24\linewidth}
			\includegraphics[width=\linewidth]{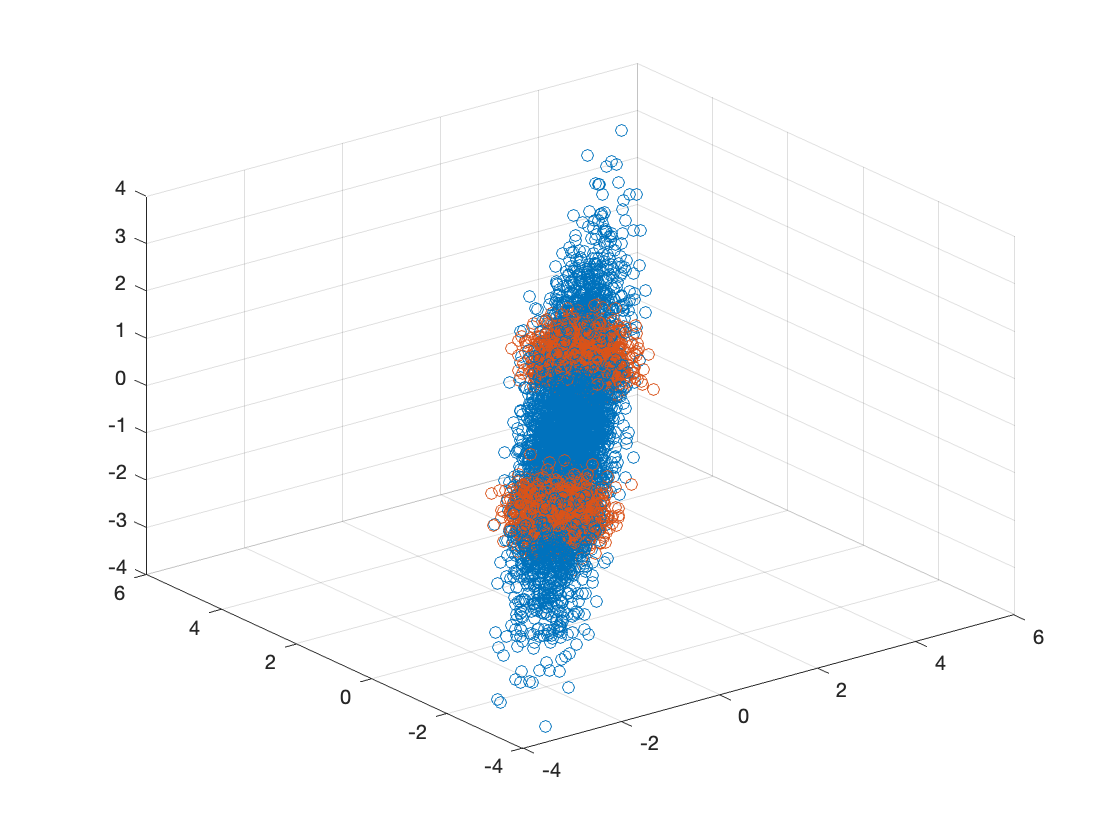}
			\caption{\label{fig:original} Original data}
		\end{subfigure}\hfill
		\begin{subfigure}[t]{0.24\linewidth}
			\includegraphics[width=\linewidth]{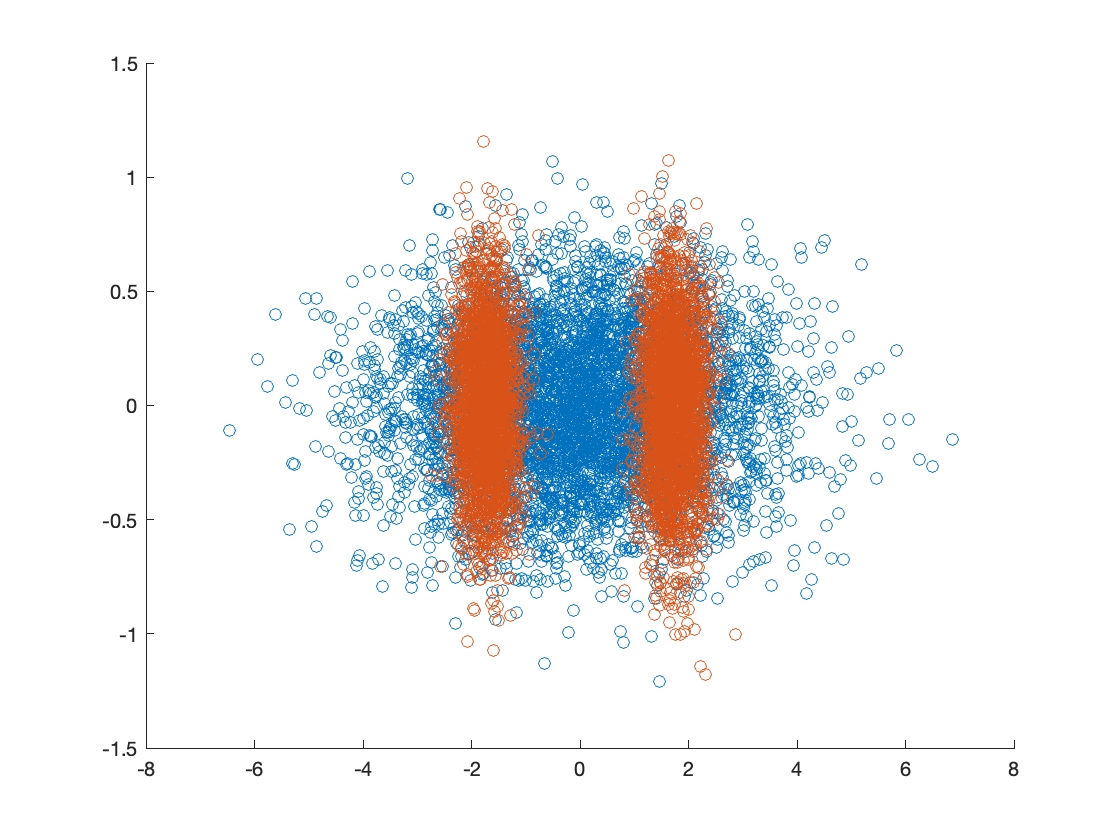}
			\caption{\label{fig:PCA} PCA}
		\end{subfigure}\hfill
		\begin{subfigure}[t]{0.24\linewidth}
			\includegraphics[width=\linewidth]{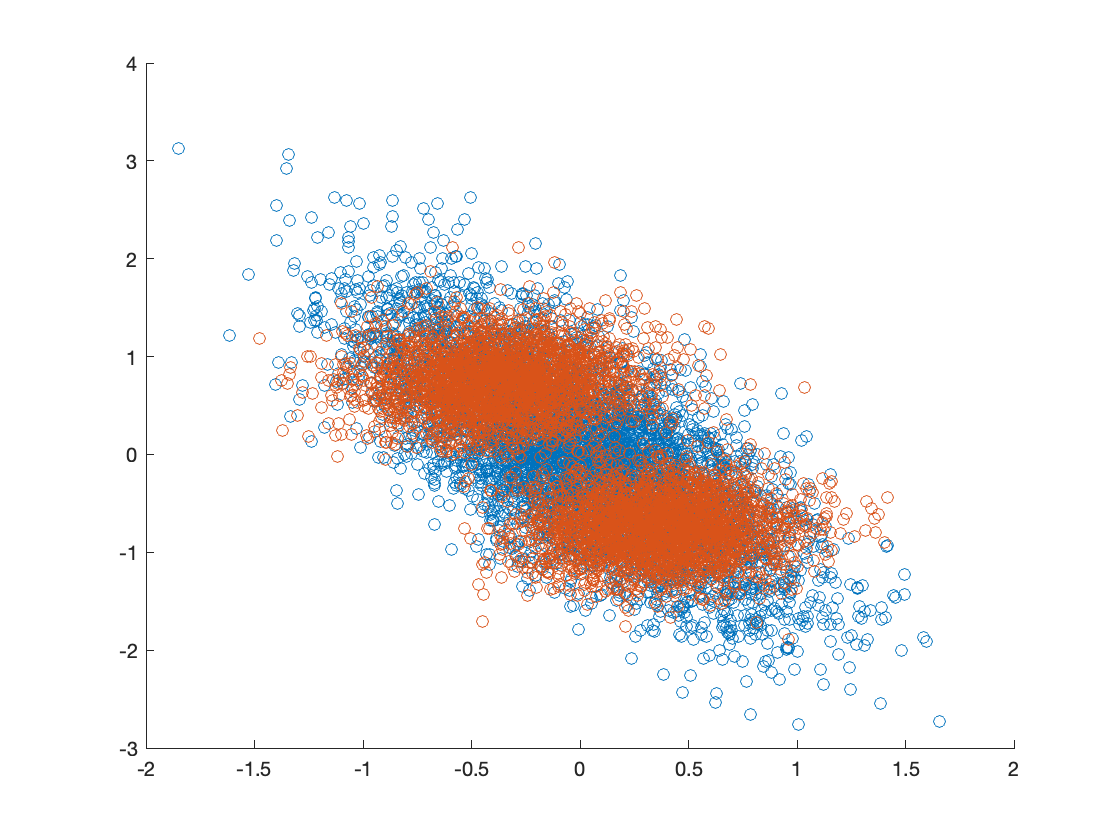}
			\caption{\label{fig:FPCA} \textsc{FPCA} \cite{OA19}}
		\end{subfigure}\hfill
		\begin{subfigure}[t]{0.24\linewidth}
			\includegraphics[width=\linewidth]{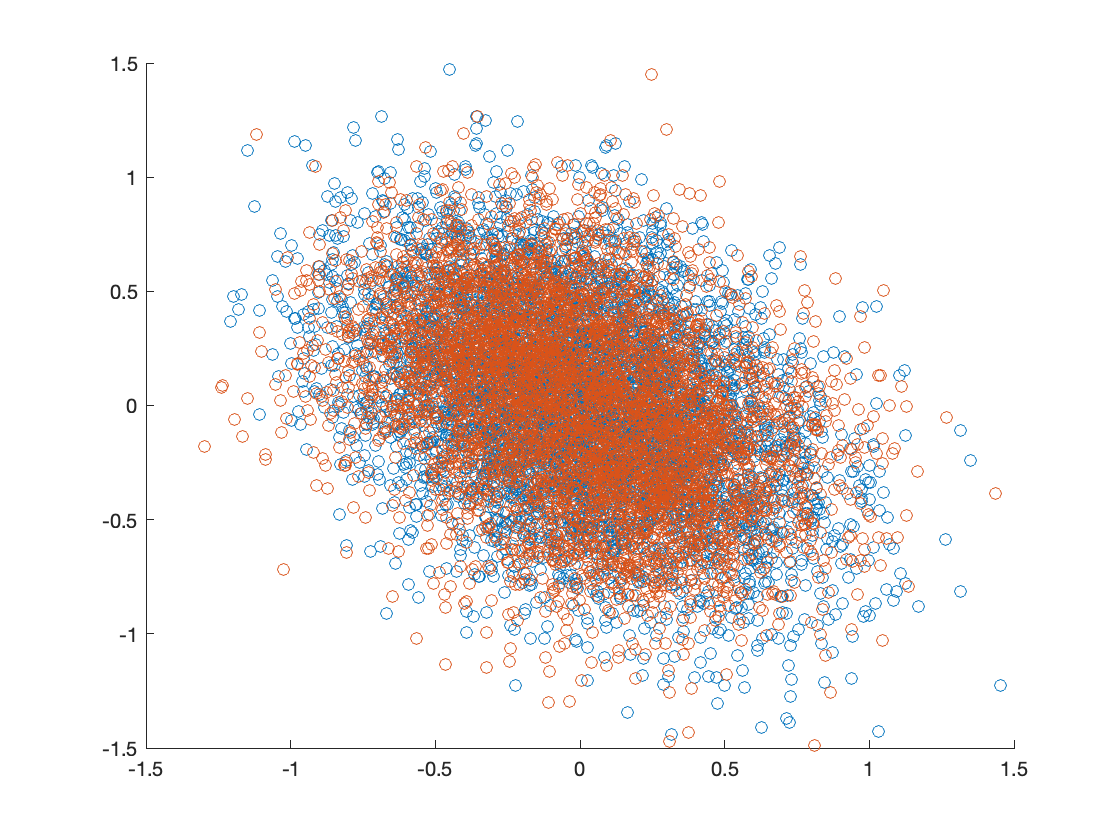}
			\caption{\label{fig:ours} \textsc{MbF-PCA} (ours)}
		\end{subfigure}
		\caption{\label{fig:exp1-1} Synthetic data \#1: Comparison of PCA, FPCA, and \textsc{MbF-PCA} on data composed of two groups with same mean and covariance, but different distributions. Blue and orange represent different protected groups.}
	\end{center}
\end{figure*}

\section{Preliminaries}
\label{sec:preliminaries}

\subsection{Notations}
For $b \geq 1$, let $\mathcal{P}_b$ be the set of all Borel probability measures defined on $\mathbb{R}^b$.
For some measurable function $\Pi : \mathbb{R}^p \rightarrow \mathbb{R}^d$ and a measure $P \in \mathcal{P}_p$, the push-forward measure of $P$ via $\Pi$ is the probability measure $\Pi_\# P \in \mathcal{P}_d$, defined as $(\Pi_\# P)(S) = P(\Pi^{-1}(S))$ for any Borel set $S$.
Let $\bf{0}$ and $\bf{1}$ denote matrices (or vectors) of zeros and ones of appropriate size, respectively.
In this work, we focus on binary cases, i.e., we assume that the protected attribute $A$ and outcome $Y$ are binary ($A, Y \in \{0, 1\}$).
We abbreviate demographic parity, equalized opportunity, and equalized odds as DP, EOP, and EOD, respectively.

\subsection{Maximum Mean Discrepancy (MMD)}

Let $k : \mathbb{R}^d \times \mathbb{R}^d \rightarrow \mathbb{R}$ be a positive-definite kernel function, and $\mathcal{H}_k$ be a unit-ball in the RKHS generated by $k$.
We impose some regularity assumptions on $k$:
\begin{assumption}
    $k$ is measurable, and bounded i.e. $K := \sup_{x,y} k(x, y) < \infty$.
\end{assumption}

Then one can pseudo-metrize $\mathcal{P}_d$ by the following distance: 
\begin{definition}[\citeauthor{Gretton07a}, \citeyear{Gretton07a}]
    Given $\mu, \nu \in \mathcal{P}_d$, their {\bf maximum mean discrepancy (MMD)}, denoted as ${\MMD}_k(\mu, \nu)$, is a pseudo-metric on $\mathcal{P}_d$, defined as follows:
    \begin{equation}
        {\MMD}_k(\mu, \nu) := \sup_{f \in \mathcal{H}_k} \left| \int_{\mathbb{R}^d} f \ d(\mu - \nu) \right|
    \end{equation}
\end{definition}

As our fairness constraint involves exactly matching the considered distributions using MMD, we require the property of $\MMD_k(\mu, \nu) = 0$ implying $\mu = \nu$.
Any kernel $k$ that satisfies such property is said to be {\bf characteristic} \cite{Fukumizu08} to $\mathcal{P}_d$.
Furthermore, \citeauthor{Sriperumbudur08} (\citeyear{Sriperumbudur08}) defined and characterized {\it stationary} characteristic kernels and identified that well-known kernels such as RBF and Laplace are characteristic.
Based on this fact,
we set $k$ to be the RBF kernel $k_{rbf}(x, y) := \exp(-\lVert x - y \rVert^2 / 2\sigma^2)$.

For the choice of bandwidth $\sigma$, the median of the set of pairwise distances of the samples after {\it vanilla PCA} is considered following the {\it median heuristic} of \cite{kernel-pca, Ramdas15a}.
For simplicity, we refer to $\MMD_{k_{rbf}}$ as $\MMD$.

\subsubsection{Benefits of MMD}
There are several reasons for using MMD as the distance on a space of probability measures. 
First, it can act as a distance between distributions with different, or even disjoint, supports.
This is especially crucial as the empirical distributions are often discrete and completely disjoint.
Such a property is not generally true, one prominent example being the KL-divergence.
Second,
since many problems in fairness involve comparing two distributions, $\MMD$ has already been used in much of the fairness literature as a metric \cite{Madras18a, Adel19} and as an explicit constraint/penalty \cite{Quadrianto17, Louizos16, Prost19, Oneto20, Jung21}, among other usages.

\subsection{Fairness for supervised learning}
The fair PCA discussed above should ultimately lead to fairness in supervised learning tasks based on the dimension-reduced data with minimal loss in performance.
Let us now review three of the most widely-used definitions of fairness in supervised learning, as formulated in \cite{Madras18a}.
Let $(Z, Y, A)\in \mathbb{R}^d \times \{0,1\} \times \{0, 1\}$ be the joint distribution of the dimensionality-reduced data, (downstream task) label, and protected attribute.
Furthermore, let $g : \mathbb{R}^d \rightarrow \{0, 1\}$ be a classifier that outputs prediction $\hat{Y}$ for $Y$ from $Z$.
We want to determine the fairness of a well-performing classifier $g$ w.r.t. protected attribute $A$.

First, let $D_s$ be the probability measure of $Z_s \triangleq Z | A = s$ for $s \in \{0, 1\}$:
\begin{definition}[\citeauthor{FeldmanFMSV15}, \citeyear{FeldmanFMSV15}]
	$g$ is said to satisfy {\bf demographic parity (DP) up to $\Delta_{DP}$} w.r.t. $A$ with 
	$\Delta_{DP} \triangleq \left| \mathbb{E}_{x \sim D_0}[g(x)] - \mathbb{E}_{x \sim D_1}[g(x)] \right|$.
\end{definition}

Now, let $D_{s, y}$ be the probability measure of $Z_s \triangleq Z | A = s, Y= y$ for $s, y \in \{0, 1\}$.

\begin{definition}[\citeauthor{HardtPNS16}, \citeyear{HardtPNS16}]
    $g$ is said to satisfy {\bf equalized opportunity (EOP) up to $\Delta_{EOP}$} w.r.t. $A$ and $Y$ with
    $\Delta_{EOP} \triangleq \left| \mathbb{E}_{x \sim D_{0, 1}}[g(x)] - \mathbb{E}_{x \sim D_{1, 1}}[g(x)] \right|$.
\end{definition}
\begin{definition}[\citeauthor{HardtPNS16}, \citeyear{HardtPNS16}]
    $g$ is said to satisfy {\bf equalized odds (EOD) up to $\Delta_{EOD}$} w.r.t. $A$ and $Y$ with
    $\Delta_{EOD} \triangleq \max_{y \in \{0, 1\}} \left| \mathbb{E}_{x \sim D_{0, y}}[g(x)] - \mathbb{E}_{x \sim D_{1, y}}[g(x)] \right|$.
\end{definition}
From hereon, we refer to such $\Delta_{f}(g)$ as the {\bf fairness metric of $f \in \{DP, EOP, EOD\}$ w.r.t. $g$}, respectively.

\section{New definition of fairness for PCA}
\label{sec:fair-dim}

For $p > d$, let $\mathbb{R}^d$ be the space onto which data will be projected.
A dimensionality reduction is a map $\Pi : \mathbb{R}^p \rightarrow \mathbb{R}^d$, and PCA is defined as $\Pi(x) = V^\intercal x$ for some $V \in \mathbb{R}^{p \times d}$ satisfying\footnote{The benefits of pursuing orthogonality in the loading matrix, and thus the resulting PCs, are already well-studied; for example, see \cite{QLZ13, BSBP16}.} $V^\intercal V = \mathbb{I}_d$ i.e. PCA is a linear, orthogonal dimensionality-reduction.
From hereon, we denote a linear PCA as the mapping $\Pi$.

Here, we consider a rather canonical notion of fairness; a projection (which includes PCA) is considered to be {\it fair} when the projected data distributions of two protected classes are the same, or at least similar.
This concept was first considered by \cite{OA19}.
As we discuss later (Section \ref{sec:relation-fpca} and \ref{sec:mbf-pca}), their definition does {\it not} directly translate into a tractable optimization formulation.
In fact, their proposed algorithm, FPCA, {\it necessarily} relies on the Gaussian assumption of the underlying data distributions i.e. FPCA only matches the means and covariances of the projected data distributions, which is problematic.
To illustrate this, consider a $3$-dimensional synthetic dataset with two protected classes colored as blue and orange, shown in Figure \ref{fig:original}.
We've designed the data distributions such that their means and covariances are almost the same.
We assume that we're reducing the dimensionality of the dataset to $2$.
Vanilla PCA results in Figure \ref{fig:PCA}, where the distribution dissimilarity persists.
Due to the Gaussian assumption, FPCA results in Figure \ref{fig:FPCA} where the dissimilarity also persists.

This calls for a new definition of fairness in PCA that can lead to a tractable form of optimization formulation that ensures fairness exactly in the sense that the equality of the two projected data distributions is guaranteed, resulting in Figure \ref{fig:ours}.
Recall how the necessity for such assumption was derived from the inherent complexity of their definition.
Inspired from our previous discussions on $\MMD$, we propose the following {\bf new} definition of fairness for PCA:
\begin{definition}[$\Delta$-fairness]
\label{def:fairness}
	Let $P_s$ be the probability measure of $X_s \triangleq X | A = s$ for $s \in \{0, 1\}$, and let $Q_s := \Pi_\# P_s \in \mathcal{P}_d$.
	Then $\Pi$ is said to be {\bf $\Delta$-fair} with $\Delta := \MMD(Q_0, Q_1)$, and we refer to $\Delta$ as the {\bf fairness metric}.
\end{definition}

In other words, low $\Delta$ means that the discrepancy between the projected conditional distributions of different protected classes, measured in a non-parametric manner using $\MMD$, while retaining as much variance as possible, is low.

Furthermore, Definition \ref{def:fairness} ensures that a downstream classification task using $\Delta$-fair dimensionality-reduced data will be fair, as formalized below\footnote{See Lemma 3 of \cite{Oneto20} for the precise statement.}:
\begin{proposition}[\citeauthor{Oneto20}, \citeyear{Oneto20}]
\label{prop:fair-representation}
	Up to a constant factor, $\MMD(Q_0, Q_1)$ bounds the MMD of the push-forward measures of $Q_0, Q_1$ via the weight vector of any given downstream task classifier $g$.
\end{proposition}

\begin{remark}
    The above discussions easily generalize to equal opportunity and equalized odds.
\end{remark}

\subsection{Relation with other definitions of fair PCA}
\label{sec:relation-fpca}

The notion of fairness proposed by \citeauthor{OA19} (\citeyear{OA19}) is similar to ours in that it measures the predictability of protected group membership in dimensonality-reduced data.
However, unlike ours, their definition is explicitly adversarial, which can be a problem.
\begin{restatable}[$\Delta_A$-fairness; \citeauthor{OA19}, \citeyear{OA19}]{definition}{testdef}
	\label{def:fairness2}
	Consider a fixed classifier $h(u, t): \mathbb{R}^d \times \mathbb{R} \rightarrow \{0, 1\}$ that inputs features $u \in \mathbb{R}^d$ and a threshold $t$, and predicts the protected class $s \in \{0, 1\}$.
	Then, $\Pi$ is  $\Delta_A(h)$-fair if
	\begin{equation}
		\begin{aligned}
			&\sup_{t \in \mathbb{R}} \Big| \mathbb{P}\big[ h(\Pi(x), t) = 1 | s = 1 \big] \\
			& - \mathbb{P}\big[ h(\Pi(x), t) = 1 | s = 0 \big] \Big| \leq \Delta_A(h)
		\end{aligned}
	\end{equation}	
	Moreover, for a family of classifiers $\mathcal{F}_c$, if $\Pi$ is $\Delta_A(h)$-fair for $\forall h \in \mathcal{F}_c$, we say that $\Pi$ is $\Delta_A(\mathcal{F}_c)$-fair.
\end{restatable}

As $\Delta_A$ can't be computed exactly, an estimator of the following form was used:
\begin{equation}
	\widehat{\Delta}_A(\mathcal{F}_c) := \sup_{h \in \mathcal{F}_c} \sup_t \left| \frac{1}{|P|} \sum_{i \in P} I_i(\Pi, h_t) - \frac{1}{|N|} \sum_{i \in N} I_i(\Pi, h_t) \right|
\end{equation}
where $\{x_i\}_{i=1}^n$ are the data points, $(P, N)$ is a partition of the index set $\{1, 2, \dots, n\}$ into two protected groups, $I_i(\Pi, h_t) = {\bf 1}(h(\Pi(x_i), t) = +1)$, and ${\bf 1}(\cdot)$ is the indicator function.

\begin{remark}
    It can be argued that, for some choice of $\mathcal{F}_c$, Definition \ref{def:fairness} and \ref{def:fairness2} are equivalent: in effect, that these are dual notions. Recognizing this, we proceed with Definition \ref{def:fairness}, as it has two main advantages in the context of our work:
    
    \begin{itemize}
        \item It ties more directly and intuitively into our optimization formulation; see Section \ref{sec:mbf-pca}.
        \item It can be represented non-variationally 
        which allows for tighter statistical guarantees.
    \end{itemize}
\end{remark}

\section{Statistical properties of $\Delta$}
\label{sec:statistical}

\subsection{Consistent and efficient estimation of $\Delta$}
\label{sec:mmd-estimator}
As defined in Definition \ref{def:fairness}, let $Q_0, Q_1 \in \mathcal{P}_d$ be the probability measures with respect to the samples of which we want to estimate $\MMD(\cdot, \cdot)$.
Let $\{X_i\}_{i=1}^m$ and $\{Y_j\}_{j=1}^n$ be these samples, respectively.
Accordingly, we consider the following estimator:
\begin{equation}
    \widehat{\Delta} := \MMD(\hat{Q}_0, \hat{Q}_1)
\end{equation}
where $\hat{Q}_s$ is the usual empirical distribution, defined as the mixture of Dirac measures on the samples.

Unlike other statistical distances (e.g. total variation distance), $\widehat{\Delta}_k$ has several desirable theoretical properties that have important practical implications, some of which we recall here.
(See \citeauthor{Sriperumbudur10a} (\citeyear{Sriperumbudur10a}) for more comprehensive discussions.)

First, $\widehat{\Delta}_k$ can be computed exactly and efficiently:
\begin{restatable}[\citeauthor{Gretton07a}, \citeyear{Gretton07a}]{lemma}{testlem}
	\label{lem:mmd-estimator}
    $\widehat{\Delta}$ is computed as follows:
	\begin{equation}
		\begin{aligned}
        \widehat{\Delta} = \Bigg[ \frac{1}{m^2} \sum_{i,j = 1}^m k(X_i, X_j) &+ \frac{1}{n^2} \sum_{i,j = 1}^n k(Y_i, Y_j) \\
		&- \frac{2}{mn} \sum_{i, j = 1}^{m, n} k(X_i, Y_j) \Bigg]^{1/2}.
		\end{aligned}
	\end{equation}
\end{restatable}

Moreover, it is asymptotically consistent with a convergence rate, depending only on $m$ and $n$:
\begin{proposition}[\citeauthor{Gretton07a}, \citeyear{Gretton07a}]
	For any $\delta > 0$, with probability at least $1 - 2\exp\left(-\frac{\delta^2 mn}{2(m + n)}\right)$ the following holds:
	\begin{equation}
	    \left| \Delta - \widehat{\Delta} \right| \leq 2 \left( \frac{1}{\sqrt{m}} + \frac{1}{\sqrt{n}} \right) + \delta
	\end{equation}
\end{proposition}

\subsection{Advantages over $\Delta_A$}
$\widehat{\Delta}_A$ is known to satisfy the following high probability bound:
\begin{proposition}[\citeauthor{OA19}, \citeyear{OA19}]
	Consider a fixed family of classifiers $\mathcal{F}_c$.
	Then for any $\delta > 0$, with probability at least $1 - \exp\left( - \frac{(n + m) \delta^2}{2} \right)$ the following holds:
	\begin{equation}
		\left| \Delta_A(\mathcal{F}_c) - \widehat{\Delta}_A(\mathcal{F}_c) \right| \leq 8 \sqrt{\frac{VC(\mathcal{F}_c)}{m + n}} + \delta
	\end{equation}
	
	where $VC(\cdot)$ is the VC dimension.
\end{proposition}

If $\mathcal{F}_c$ is too expressive in terms of VC-dimension, then the above bound may become void. This is the case, for instance, when $\mathcal{F}_c$ is the set of RBF-kernel SVMs.

In addition, computing $\widehat{\Delta}_A$ requires considering all possible classifiers in the designated family $\mathcal{F}_c$.
This is computationally infeasible, and it forces one to use another approximation (e.g. discretization of $\mathcal{F}_c$), which incurs additional error that may further inhibit asymptotic consistency.

As exhibited in the previous subsection, our $\MMD$-based approach suffers from none of these issues.

\section{Manifold optimization for \textsc{MbF-PCA}}
\label{sec:mbf-pca}

\subsection{Improvements over \textsc{FPCA}}
\citeauthor{OA19} (\citeyear{OA19}) proposed \textsc{FPCA}, an SDP formulation of fair PCA\footnote{See Section \ref{sec:fpca} of the supplementary material (SP) for its precise description.}, in which matching the first and second moments of the protected groups after dimensionality-reduction are approximated as convex constraints.
However, this has several shortcomings, which we discuss here and empirically exhibit in a later section.

First, there are cases in which matching the mean and covariance alone is not enough; this was explicitly shown and discussed in Section \ref{sec:fair-dim} where we've presented a simple ``counterexample'' in which two protected groups have different distributions with same mean and covariance. (Figure \ref{fig:exp1-1})
While this previous point may be countered by the application of the kernel trick to FPCA, this raises a second issue: their formulation requires solving \footnote{In their approach, the final solution $V$ is obtained by taking the first $d$ eigenvectors of $P$.} a $p \times p$-dimensional SDP, motivated by the reparameterization  $P = V V^\intercal$ \cite{Arora13}.
Since SDP is known to become inefficient (or even computationally infeasible) in high dimensions, this quickly becomes intractable for high-dimensional data (for linear or polynomial kernels) or for any moderate to large size datasets (for the RBF kernel).
Finally, their approach involves a relaxation of a rank constraint ($\rank(P) \leq d$) to a trace constraint ($\tr(P) \leq d$), yielding sub-optimal outputs in presence of (fairness) constraints, even to substantial order in some cases.
In Section \ref{sec:low-sdp} of the SP, we discuss in detail why \textsc{FPCA} may lead to such degraded explained variance.

\subsection{Formulating \textsc{MbF-PCA}}
Observing that the shortcomings of \textsc{FPCA} stem from the reparametrization of $P = VV^\intercal$, we propose a new formulation of fair PCA that solves {\it directly} for $V$.
This allows for an effective and efficient approach.

We start by noting that the set of all $V$ with orthonormal columns has the intrinsic geometric structure of a {\it manifold}:
\begin{definition}
	For $p > d$, the Stiefel manifold, denoted as $St(p, d)$, is an embedded Riemannian sub-manifold of $\mathbb{R}^{p \times d}$ such that each element of $St(p, d)$ has orthonormal columns i.e. $V^\intercal V = I_d$ for all $V \in St(p, d)$. 
\end{definition}

$St(p, d)$ has several desirable properties such as compactness, completeness and smoothness, which we present in Section \ref{sec:background} of the SP.
As $St(p, d)$ is prevalent in various fields of machine learning (most notably PCA), much work has been done that focuses on exploiting this geometric structure for efficient optimization \cite{HLWY20}.


Based on our MMD-based formulation and letting $\Sigma$ be the sample covariance matrix of the full dataset, we formulate our fair PCA as follows, which we refer to as \textsc{MbF-PCA}:
\begin{equation}
\label{eq:MbF-PCA}
\begin{aligned}
	& \underset{V \in St(p, d)}{\text{minimize}}
	& &  f(V) := - \langle \Sigma, V V^\intercal \rangle\\
	& \text{subject to}
	& & h(V) := {\MMD}^2(Q_0, Q_1) = 0.
\end{aligned}
\end{equation}
Here, $Q_0$ and $Q_1$ are defined as in definition \ref{def:fairness}.
Observe how our definition of fairness {\it directly} incorporates itself into the optimization problem as a constraint.

\begin{remark}
    Under Gaussian assumption, our MMD-based formulation indeed amounts to the same constraints as \textsc{FPCA} since $\MMD^2(\cdot, \cdot)$ is a metric and Gaussian distribution is completely characterized by its first and second moments.
\end{remark}


\section{REPMS for \textsc{MbF-PCA}, with new theoretical guarantees}
\label{sec:repms}

\subsection{Description of Algorithm \ref{alg:mbfpca}}
One crucial observation is that the constraint function $h$ is always non-negative\footnote{This follows from closed form of $\widehat{\Delta}_k$ (Lemma \ref{lem:mmd-estimator}). See Section 2 of \citeauthor{Gretton12a} (\citeyear{Gretton12a}) for more details.} and smooth.
This motivates the use of the exact penalty method \cite{exact-penalty}, recently extended to manifold optimization as the Riemannian Exact Penalty Method via Smoothing (REPMS; \citeauthor{repms}, \citeyear{repms}).
Note that smoothing tricks \cite{repms}, which were required to smooth out possible non-differentiable functions emerging from the $\ell_1$-penalty, are {\it not} necessary.
Moreover, by leveraging the kernel trick, there is a {\it closed} form for $\grad_V h(V)$, thus alleviating the need for computationally expensive numerical approximations; see Section \ref{sec:closed-form} of the SP for the derivation.
The pseudo-code for the algorithm is shown in Algorithm \ref{alg:mbfpca}.

\begin{algorithm}[!th]
	\SetAlgoLined
	\KwInput{$X$, $K$, $\epsilon_{min}, \epsilon_0 > 0$,  $\theta_\epsilon \in (0, 1)$, $\rho_0 > 0$, $\theta_\rho > 1$, $\rho_{max} \in (0, \infty)$, $\tau > 0$, $d_{min} > 0$.}
	
    Initialize $V_0$\;
    
	\For{$k = 0, 1, \dots, K$}{
		
		Compute an approximate solution $V_{k+1}$ for the following sub-problem, with a warm-start at $V_k$, until $\lVert \mathrm{grad} \ \mathcal{Q} \rVert \leq \epsilon_k$:
		\begin{equation}
		\label{eq:subprob}
			\min_{V \in St(p, d)} \mathcal{Q}(V, \rho_k)
		\end{equation}
		where
		\begin{equation*}
			\mathcal{Q}(V, \rho_k) = f(V) + \rho_k h(V)
		\end{equation*}
		
		\If{$\lVert V_{k+1} - V_k \rVert_F \leq d_{min}$ and $\epsilon_k \leq \epsilon_{min}$}{
			\If{ $h(V_{k+1}) \leq \tau$}{
				\Return{$V_{k+1}$}\;
		}
		}
		$\epsilon_{k+1} = \max\{ \epsilon_{min}, \theta_\epsilon \epsilon_k \}$\;
		
		\eIf{$h(V_{k+1}) > \tau$}{
				$\rho_{k+1} = \min(\theta_\rho \rho_k, \rho_{max})$\;
		}{
			$\rho_{k+1} = \rho_k$\;
		}
	}
	\caption{REPMS for MbF-PCA}
	\label{alg:mbfpca}
\end{algorithm}

For practical concerns that will be addressed in the following subsection, we've set the {\it fairness tolerance level}, $\tau$, to be a fixed and sufficiently small, non-negative value.
Formally, we consider the following definition:
\begin{definition}
\label{def:approx-fair}
	For fixed $\tau \geq 0$, $V \in St(p, d)$ is {\bf $\tau$-approximate fair} if it satisfies $h(V) \leq \tau$.
	If $\tau = 0$, we simply say that $V$ is {\bf fair}.
\end{definition}

\subsection{New theoretical guarantees for Algorithm 1}
We start by observing that Eq. \eqref{eq:subprob} in Algorithm 1 is smooth, unconstrained manifold optimization problem, which can be solved using conventional algorithms; these include first-order methods like line-search methods \cite{AMS07}, or second-order methods like the Riemannian Trust Region (RTR; \citeauthor{ABG07}, \citeyear{ABG07}) method.
It is known that, pathological examples excluded, most conventional {\it unconstrained} manifold optimization solvers produce iterates whose limit points are local minima, and not other stationary points such as saddle point or local maxima: see \cite{ABG07, AMS07} for more detailed discussions.

Motivated by this, we consider the following assumption:
\begin{assumption}[informal; locality assumption]
\label{assumption:1}
	Each $V_{k+1}$ is sufficiently close to a local minimum of Eq. \eqref{eq:subprob}.
\end{assumption}
Lastly, we consider the following auxiliary optimization problem:
\begin{equation}
	\label{eq:auxiliary}
		\min_{V \in St(p, d)} h(V)
\end{equation}

The following theorem, whose proof is deferred to Section \ref{sec:proofs} of the SP, provides an {\it exact} theoretical convergence guarantee of \textsc{MbF-PCA} under the {\it ideal} hyperparameter setting:
\begin{theorem}
\label{thm:mbfpca}
	Let $K = \infty$, $\rho_{max} = \infty$, $\epsilon_{min} = \tau = 0$, $\{V_k\}$ be the sequence generated by Alg. \ref{alg:mbfpca} under Assumption \ref{assumption:1}, and  $\overline{V}$ be any limit point of $\{V_k\}$, whose existence is guaranteed. Then the following holds:
	\begin{enumthm}
	    \item $\overline{V}$ is a local minimizer of Eq. \eqref{eq:auxiliary}, which is a necessary condition for $\overline{V}$ to be fair.
	    
	    \item If $\overline{V}$ is fair, then {$\overline{V}$ is a local minimizer of Eq. \eqref{eq:MbF-PCA}}
	\end{enumthm}
\end{theorem}

The assumption of $\overline{V}$ being fair, which is used in (B), is at least partially justified in (A) in the following sense: 
the ideal hyperparameter setting of $\rho_{max} = \infty, \tau = 0, \epsilon_{min} = 0$ implies the {\it exact} local minimality of $\overline{V}$ for Eq. \eqref{eq:auxiliary}, which is in turn a {\it necessary condition} for $\overline{V}$ to be fair.

The next theorem, whose is also deferred to Section \ref{sec:proofs} of the SP, asserts that with small $\tau, \epsilon_{min}$ and large $\rho_{max}$, the above guarantee can be approximated in rigorous sense:
\begin{theorem}
\label{thm:mbfpca-another}
	Let $K = \infty$, $\rho_{max} < \infty$, $\epsilon_{min}, \tau > 0$, $\{V_k\}$ be the sequences generated by Alg. \ref{alg:mbfpca} under Assumption \ref{assumption:1} and $\overline{V}$ be any limit point of $\{V_k\}$, whose existence is guaranteed. 
	Then for any sufficiently small $\epsilon_{min}$ and  $\tilde{r} = \tilde{r}(\epsilon_{min}) > 0$,
	the following hold:
	
	\begin{enumthm}
	    \item $\overline{V}$ is an approximate local minimizer of Eq. \eqref{eq:auxiliary} in the sense that
	    \begin{equation}
	        h(\overline{V}) \leq h(V) + \beta \lVert V - \overline{V} \rVert + (\beta + L_h) g(\epsilon_{min})
	    \end{equation}
	    
	    for all $V \in B_{\tilde{r}}(\overline{V}) \cap St(p, d)$, where $\beta = \beta(\rho_{max}, \tau)$ is a function that satisfies the following:
	    \begin{itemize}
	        \item $0 < \beta \leq \frac{2 \lVert \Sigma \rVert}{\rho_0}$
	        
	        \item $\beta(\rho_{max}, \tau)$ is increasing in $\rho_{max}$ and decreasing in $\tau$.
	    \end{itemize}
	    
		\item If $\overline{V}$ is fair, then it is an approximate local minimizer of Eq. \eqref{eq:MbF-PCA} in the sense that it satisfies
		\begin{equation}
			f(\overline{V}) \leq f(V) + 2 \lVert \Sigma \rVert g(\epsilon_{min})
		\end{equation}
		
		for all fair $V \in B_{\tilde{r}}(\overline{V}) \cap St(p, d)$.
	\end{enumthm}
	
	In both (A) and (B), $g$ is some continuous, decreasing function that satisfies $g(0) = 0$, and $\tilde{r}(\epsilon_{min}) = r - g(\epsilon_{min})$ for some fixed constant $r > 0$.
\end{theorem}

Existing optimality guarantee of REPMS (Proposition 4.2; \citeauthor{repms}, \citeyear{repms}) states that when $\epsilon_{min} = 0$, $\rho$ is {\it not} updated (i.e. line 10-14 is ignored), and the resulting limit point is feasible, then that limit point satisfies the KKT condition \cite{optimality-manifold}.
Comparing Theorem \ref{thm:mbfpca} and \ref{thm:mbfpca-another} to the previous result, we see that ours extend the previous result in several ways:
\begin{itemize}
    \item Our theoretical analyses are much closer to the actual implementation, by incorporating the $\rho$-update step (line 11) and the {\it practical} hyperparameter setting.
    
    \item Our theoretical analyses are much more stronger in the sense that 1) by {\it introducing} a reasonable, yet novel locality assumption, we go beyond the existing KKT conditions and prove the {\it local minimality} of the limit point, and 2) we provide a partial justification of the feasibility assumption in (A) by proving a necessary condition for it.
\end{itemize}

\subsection{Practical implementation}
\label{sec:practical}
    
    

In line 4 in Algorithm 1, we implemented the termination criteria: sufficiently small distance between iterates and sufficiently small tolerance for solving Eq. \eqref{eq:subprob}.
However, such a heuristic may return some point $\overline{V}$ that is not $\tau$-approximate fair for user-defined level $\tau$ in practical hyperparameter setting.
To overcome this issue, we've additionally implemented line 5 that forces the algorithm to continue on with the loop until the desired level of fairness is achieved.


\section{Related Work}
\label{sec:related-works}

\subsection{Fairness in ML}
A large body of work regarding fairness in the context of supervised learning \cite{FeldmanFMSV15, CaldersKP09, DworkHPRZ12, HardtPNS16, ZafarVGG17a} has been published.
This includes key innovations in quantifying algorithmic bias, notably the concepts of \textit{demographic parity} and \textit{equalized odds (opportunity)} that have become ubiquitous in fairness research \cite{BarocasS16, HardtPNS16}.
More recently, fair machine learning literatures have branched out into a variety of fields, including deep learning \cite{beutel2017data}, regression \cite{calders2013controlling}, and even hypothesis testing \cite{olfat2020covariance}.

Among these, one line of research has focused on learning fair representations \cite{KamiranC11, ZemelWSPD13, FeldmanFMSV15, CalmonWVRV17}, which aims to learn a representation of the given data on which various fairness definitions are ensured for downstream modeling.
A growing number of inquiries have been made into highly specialized algorithms for specific unsupervised learning problems like clustering \cite{ChierichettiKLV17, KleindessnerAM19, BeraCFN19}, but these lack the general applicability of key dimensionality reduction algorithms such as PCA \cite{Pearson1901, Hotelling1933}.

To the best of our knowledge, Olfat \& Aswani (\citeyear{OA19}) is the {\bf only} work on incorporating fair representation to PCA, making it the sole comparable approach to ours.
Another line of work \cite{Bian11, SamadiTMSV18, TantipongpipatS19, Zalcberg21} considers a completely {\it orthogonal} definition of fairness for PCA: minimizing the discrepancy between reconstruction errors over protected attributes.
This doesn't ensure the fairness of downstream tasks, rendering it incomparable to our definition of fairness; see Section \ref{sec:compare} of the SP for more details.

\subsection{Manifold Optimization}
A constrained problem over Euclidean space can be transformed to an unconstrained problem over a manifold (or at least manifold optimization with less constraints).
Many algorithms for solving Euclidean optimization problems have direct counterparts in manifold optimization problems that includes Riemannian gradient descent and Riemannian BFGS.
By making use of the geometry of lower dimensional manifold structure, often embedded in potentially very high dimensional ambient space, such Riemannian counterparts are much more computationally {\it efficient} than algorithms that do not make use of manifold structure.
This is shown in numerous literatures \cite{repms, Alsharif21, Meng21}, including this work.
We refer interested readers to the standard textbooks \citeauthor{AMS07} (\citeyear{AMS07}) and \citeauthor{Boumal20} (\citeyear{Boumal20}) on this field, along with a survey by \citeauthor{HLWY20} (\citeyear{HLWY20}).

\section{Experiments}
\label{sec:experiments}
\begin{figure}[!t]
	\begin{center}
		\begin{subfigure}[t]{0.49\columnwidth}
			\includegraphics[width=\linewidth]{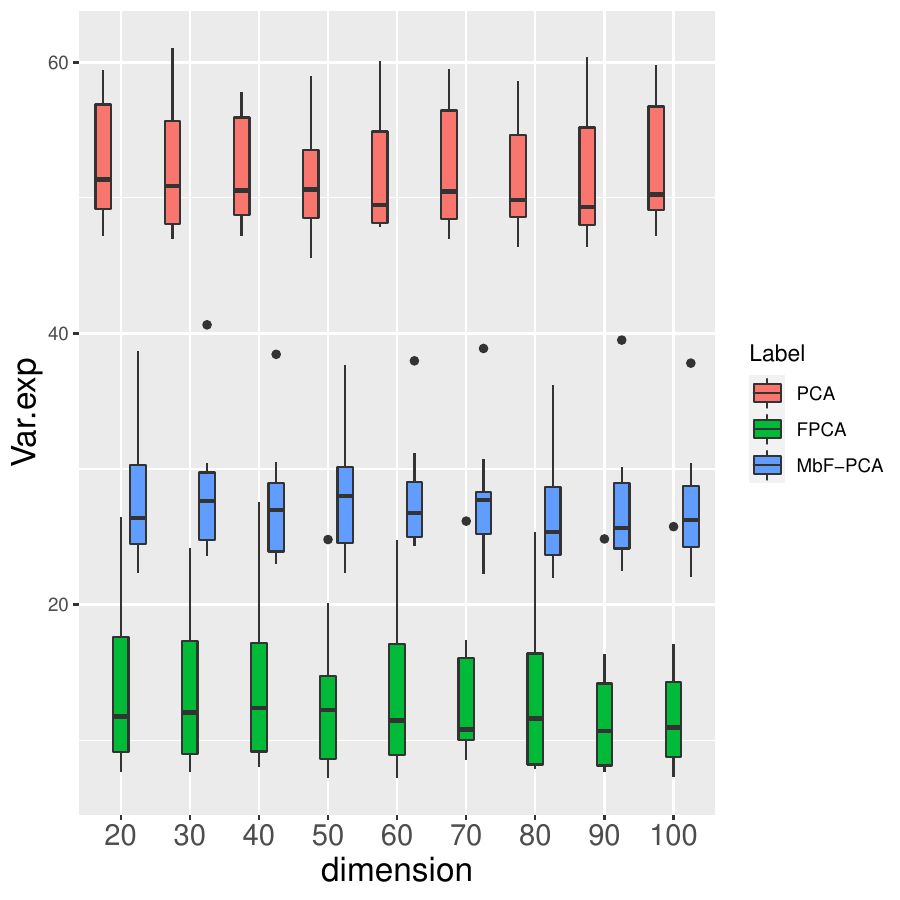}
			\caption{\label{fig:varexp} Variance explained (\%)}
		\end{subfigure}
		\begin{subfigure}[t]{0.49\columnwidth}	
			\includegraphics[width=\linewidth]{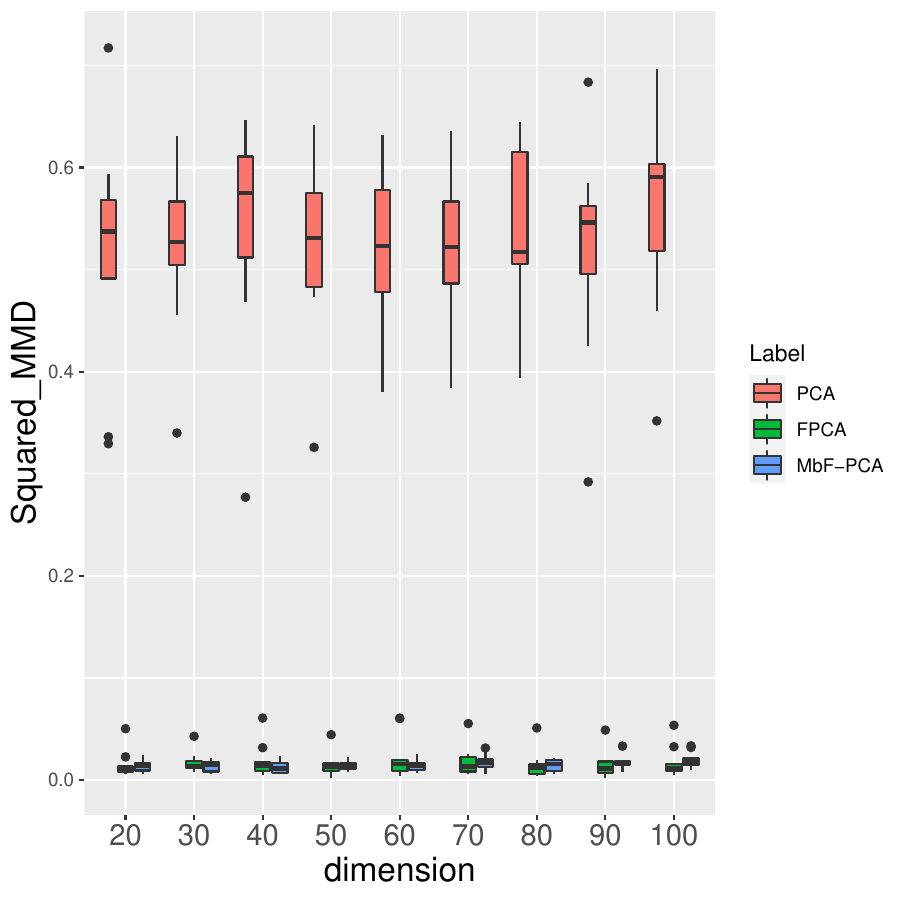}
			\caption{\label{fig:mmd} $\MMD^2$}
		\end{subfigure}
	\end{center}
	\caption{\label{fig:exp1-2} Synthetic data \#2: Comparison of PCA, FPCA, and \textsc{MbF-PCA} on the synthetic datasets of increasing dimensions.}
\end{figure}

\begin{figure}[!t]
	\begin{center}
		\includegraphics[width=0.7\linewidth]{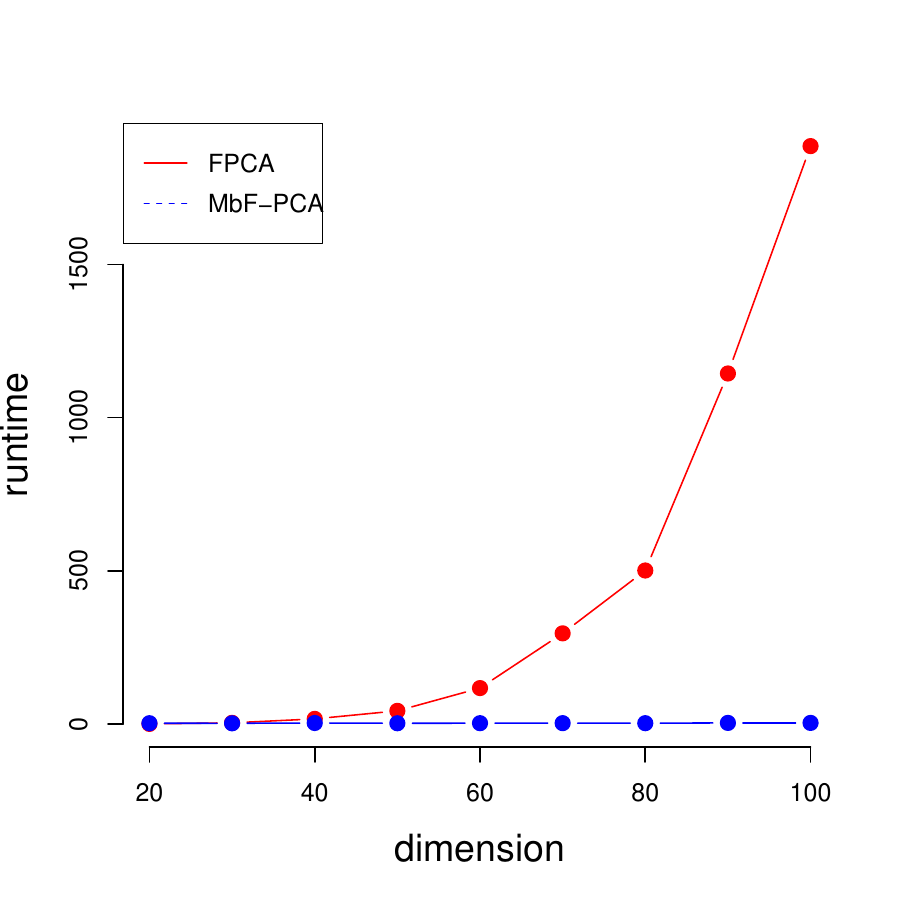}
	\end{center}
	\caption{\label{fig:exp1-2-runtime} Synthetic data \#2: Comparison of runtimes of FPCA, and MbF-PCA.}
\end{figure}
\begin{table*}[!t]
	\caption{Comparison of PCA, FPCA, \textsc{MbF-PCA} for UCI datasets. Number in parenthesis for each dataset is its dimension.
	Also, the parenthesis for each fair algorithm is its hyperparameter setting; $(\mu, \delta)$ for \textsc{FPCA} and $\tau$ for \textsc{MbF-PCA}.
	Among the fair algorithms considered, results with the best mean values are {\bf bolded}.
	Results in which our approach terminates improperly in the sense that the maximum iteration is reached before passing the termination criteria are \colorbox{lightgray}{highlighted}.
	}
	\label{tab:results}
	\begin{center}
		\begin{small}
			\begin{sc}
			\resizebox{\linewidth}{!}{
				\begin{tabular}{l|r|cccc|cccc|cccr}
					\toprule
					 & & \multicolumn{4}{c}{COMPAS ($11$)} & \multicolumn{4}{c}{German Credit ($57$)} & \multicolumn{4}{c}{Adult Income ($97$)} \\
					 $d$ & Alg. & \%Var & \%Acc & $\MMD^2$ & $\Delta_{DP}$ & \%Var & \%Acc & $\MMD^2$ & $\Delta_{DP}$ & \%Var & \%Acc & $\MMD^2$ & $\Delta_{DP}$ \\
					\midrule
					
                    & PCA
                    & $39.28_{5.17}$ & $64.53_{1.45}$ & $0.092_{0.010}$ & $0.29_{0.09}$
                    & $11.42_{0.47}$ & $76.87_{1.39}$ & $0.147_{0.049}$ & $0.12_{0.06}$
                    & $7.78_{0.82}$ & $82.03_{1.15}$ & $0.349_{0.027}$ & $0.20_{0.05}$ \\
                    
                    & \textsc{FPCA} ($0.1, 0.01$)
                    & $\bf 35.06_{5.16}$ & $61.65_{1.17}$ & $0.012_{0.007}$ & $0.10_{0.07}$
                    & $7.43_{0.59}$ & $72.17_{1.09}$ & $0.017_{0.010}$ & $0.03_{0.02}$
                    & $4.05_{0.98}$ & $77.44_{2.96}$ & $0.016_{0.011}$ & $0.04_{0.04}$ \\

                    $2$ & \textsc{FPCA} ($0, 0.01$)
                    & $34.43_{5.02}$ & $60.86_{1.09}$ & $0.011_{0.006}$ & $0.10_{0.06}$
                    & $7.33_{0.57}$ & $71.77_{1.60}$ & $\bf 0.015_{0.010}$ & $0.03_{0.03}$
                    & $3.65_{0.97}$ & $77.05_{3.18}$ & $\bf 0.005_{0.004}$ & $\bf 0.01_{0.01}$ \\
                    
                    & \textsc{MbF-PCA} ($10^{-3}$)
                    & $33.95_{5.01}$ & $\bf 65.37_{1.11}$ & $0.005_{0.002}$ & $0.12_{0.07}$
                    & $\bf 10.17_{0.57}$ & $\bf 74.53_{1.92}$ & $0.018_{0.014}$ & $0.05_{0.04}$
                    & $\bf 6.03_{0.61}$ & $\bf 79.50_{1.22}$ & $\bf 0.005_{0.004}$ & $0.03_{0.02}$ \\
                    
                    & \textsc{MbF-PCA} ($10^{-6}$)
                    & $11.83_{3.59}$ & $57.73_{1.50}$ & $\bf 0.002_{0.002}$ & $\bf 0.06_{0.08}$
                    & $9.36_{0.33}$ & $74.10_{1.56}$ & $0.016_{0.010}$ & $\bf 0.02_{0.02}$
                    & $5.83_{0.57}$ & $79.12_{1.14}$ & $\bf 0.005_{0.004}$ & $\bf 0.01_{0.01}$ \\
					
					\midrule
					
					& PCA & $100.00_{0.00}$ & $73.14_{1.22}$ & $0.241_{0.005}$ & $0.21_{0.07}$ 
					& $38.25_{0.98}$ & $99.93_{0.14}$ & $0.130_{0.019}$ & $0.12_{0.08}$
					& $21.77_{2.06}$ & $93.64_{0.92}$ & $0.195_{0.007}$ & $0.16_{0.01}$ \\
                    
                    & \textsc{FPCA} ($0.1, 0.01$)
                    & $\bf 87.79_{1.27}$ & $72.25_{0.93}$ & $0.015_{0.003}$ & $\bf 0.16_{0.06}$
                    & $29.85_{0.87}$ & $\bf 99.93_{0.14}$ & $0.020_{0.005}$ & $0.12_{0.08}$
                    & $15.75_{1.20}$ & $91.94_{0.88}$ & $0.006_{0.003}$ & $0.13_{0.02}$ \\
                    
                    $10$ & \textsc{FPCA} ($0, 0.1$)
                    & $87.44_{1.35}$ & $\bf 72.32_{0.93}$ & $0.015_{0.002}$ & $\bf 0.16_{0.07}$
                    & $29.79_{0.89}$ & $\bf 99.93_{0.14}$ & $0.020_{0.006}$ & $0.12_{0.08}$
                    & $15.52_{1.18}$ & $91.66_{0.97}$ & $0.004_{0.002}$ & $0.13_{0.02}$ \\
                    
                    & \textsc{MbF-PCA} ($10^{-3}$)
                    & \mycc$87.75_{1.36}$ & \mycc$72.16_{0.90}$ & \mycc$\bf 0.014_{0.002}$ & \mycc$\bf 0.16_{0.07}$
                    & $\bf 34.10_{1.00}$ & $\bf 99.93_{0.14}$ & $0.020_{0.008}$ & $0.12_{0.08}$
                    & $\bf 18.71_{1.47}$ & $\bf 92.81_{0.84}$ & $0.005_{0.002}$ & $0.14_{0.01}$ \\
                    
                    & \textsc{MbF-PCA} ($10^{-6}$)
                    & \mycc$87.75_{1.36}$ & \mycc$72.16_{0.90}$ & \mycc$\bf 0.014_{0.002}$ & \mycc$\bf 0.16_{0.07}$
                    & $16.95_{1.52}$ & $92.70_{3.00}$ & $\bf 0.013_{0.007}$ & $\bf 0.06_{0.05}$
                    & \mycc$15.49_{6.44}$ & \mycc$86.36_{3.77}$ & \mycc$\bf 0.003_{0.002}$ & \mycc$\bf 0.07_{0.03}$ \\
					\bottomrule
				\end{tabular}
				}
			\end{sc}
		\end{small}
	\end{center}
\end{table*}

\textsc{MbF-PCA} was implemented using ROPTLIB \cite{roptlib}, a state-of-the-art manifold optimization framework on MATLAB.
For solving Eq. \eqref{eq:subprob}, we use the cautious Riemannian BFGS method (RBFGS; \citeauthor{HuangAG18}, \citeyear{HuangAG18}), a memory-efficient quasi-Newton method.
As for the hyperparameters, we've set $K = 100, \epsilon_{min} = 10^{-6}, \epsilon_0 = 10^{-1}, \theta_\epsilon = (\epsilon_{min} / \epsilon_0)^{1/5}, \rho_{max} = 10^{10}, \theta_{\rho} = 2, d_{min} = 10^{-6}$.
For \textsc{FPCA}, we use the same Python MOSEK\cite{mosek} implementation as provided by \cite{OA19}.
$(\mu, \delta)$ are the hyperparameters of \textsc{FPCA}; see Section \ref{sec:fpca} of the SP.
Codes are available in our Github repository\footnote{\url{https://github.com/nick-jhlee/fair-manifold-pca}}.

All data is pre-processed to be standardized such that each covariate has zero mean and unit variance.
For all experiments, we considered $10$ different $70-30$ train-test splits.

\subsection{Synthetic data \#1}
We consider synthetic data composed of two groups, each of size $n=150$; one is sampled from $\mathcal{N}_3(\mathbf{0}, 0.1I_3 + \mathbf{1})$ and one is sampled from a (balanced) mixture of $\mathcal{N}_3(\mathbf{1}, 0.1I_3)$ and $\mathcal{N}_3(-\mathbf{1}, 0.1I_3)$.
Note how the two groups follow different distributions, yet have the same mean and covariance.
Thus, we expect \textsc{FPCA} to project in a similar way as vanilla PCA, while \textsc{MbF-PCA} should find a fairer subspace such that the projected distributions are exactly the same.
Hyperparameters are set as follows: $\delta = 0, \mu = 0.01$ for \textsc{FPCA} and $\tau = 10^{-5}$ for \textsc{MbF-PCA}.
We've set $d = 2$ and Figure \ref{fig:exp1-1} displays the results of each algorithm using the top two principal components.
Indeed, only \textsc{MbF-PCA} successfully obfuscates the protected group information by merging the two orange clusters with the blue cluster.

\subsection{Synthetic data \#2}

We consider a series of synthetic datasets of dimension $p$.
For each $p$, the dataset is composed of two groups, each of size $n=240$ and sampled from two different $p$-variate normal distributions.
We vary $p \in \{20, 30, \dots, 100\}$; see Section \ref{sec:exp1-description} of the SP for a full description of the setting.
For the hyperparameters, we've set $\delta = 0, \mu = 0.01$ for \textsc{FPCA} and $\tau = 10^{-5}$ for \textsc{MbF-PCA}.

Figure \ref{fig:exp1-2} plots the explained variance and fairness metric values.
Observe how \textsc{MbF-PCA} achieves better explained variance, while achieving similar level of fairness.
In addition, Figure \ref{fig:exp1-2-runtime} shows a clear gap in runtime between \textsc{FPCA} and \textsc{MbF-PCA}; the runtime of \textsc{FPCA} explodes for even moderate problem sizes, while \textsc{MbF-PCA} scales well.
For higher dimensions, conventional computing machine will not be able to handle such computational burden.

\subsection{UCI datasets}
\begin{figure}[!t]
	\begin{center}
		\includegraphics[width=\linewidth]{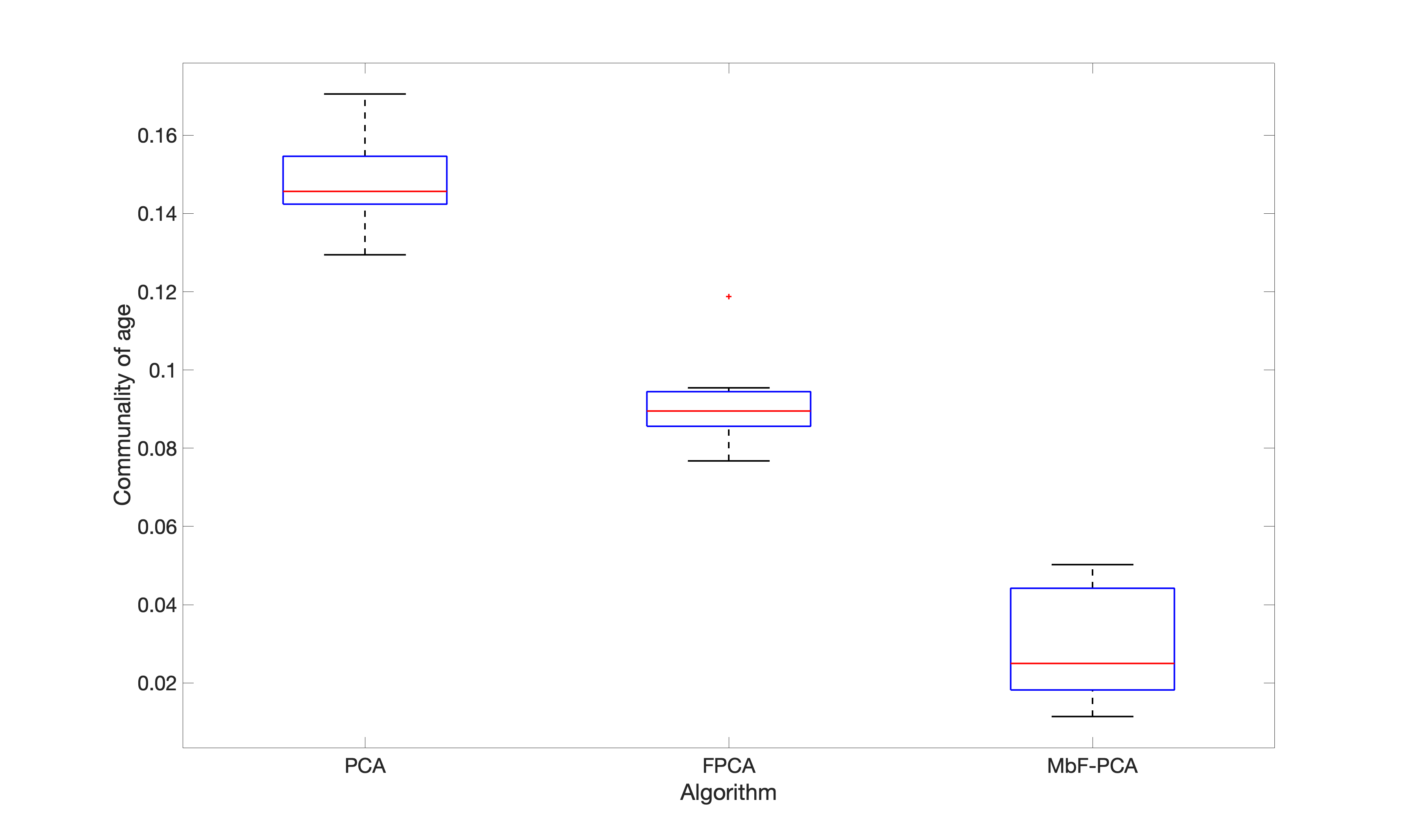}
	\end{center}
	\caption{\label{fig:exp2-2-eda} Comparison of communality of ``age" of German credit dataset for PCA, \textsc{FPCA}, and \textsc{MbF-PCA}.}
\end{figure}

For target dimensions $d \in \{2, 10\}$, we compare the performance of \textsc{FPCA} and \textsc{MbF-PCA} on $3$ datasets from the UCI Machine Learning Repository \cite{UCI}; COMPAS dataset \cite{angwin2016machine}, Adult income dataset, and German credit dataset.
See Section \ref{sec:exp2-description} of the SP for complete description of the pre-processing steps.
For both algorithms, we consider two different hyperparameter settings, such that one simulates the relaxed fairness while the other simulates a stricter fairness constraints.
For computing $\Delta_{DP}(g)$, we trained a RBF SVM $g$ to be the {\it downstream task classifier} that best classifies the target attribute in the {\it dimensionality-reduced} data.
Table \ref{tab:results} displays the results, from which several observations can be made:
\begin{itemize}
    \item Across all considered datasets, \textsc{MbF-PCA} is shown to outperform \textsc{FPCA} in terms of fairness (both $\MMD^2$ and $\Delta_{DP}$) with low enough $\tau$.
    
    \item For \textsc{German Credit} and \textsc{Adult Income}, \textsc{MbF-PCA} shows a clear trade-off between explained variance and fairness; by relaxing $\tau$, we see that \textsc{MbF-PCA} outperforms \textsc{FPCA} in terms of explained variance and downstream task accuracy.
\end{itemize}

In addition, to see how correlated are the PCs with the protected attribute, we examine the communalities.
For clarity of exposition, we consider the German credit dataset, whose protected attribute is age, and $d = 10$.
Here, we again consider PCA, \textsc{FPCA} ($0, 0.01$), and \textsc{MbF-PCA} ($10^{-3}$).
For PCA, communality of a feature is its variance contributed by the PCs \cite{multivariate-analysis}, computed as the sum of squares of the loadings of the considered feature.
High communality implies that the correlations between the feature and PCs are strong.
Figure \ref{fig:exp2-2-eda} displays the boxplot of communality, plotted over the $10$ splits.
Indeed the amount of variance in age that is accounted for from the loadings of \textsc{MbF-PCA} is much smaller than that of PCA or \textsc{FPCA}.
In other words, the PCs resulting from \textsc{MbF-PCA} have the {\it least} correlations with age, the protected attribute.

\section{Conclusion and Future Works}
\label{sec:conclusion}
We present a MMD-based definition of fair PCA, and formulate it as a constrained optimization over the Stiefel manifold.
Through both theoretical and empirical discussions, we show that our approach outperforms the previous approach \cite{OA19} in terms of explained variance, fairness, and runtime.
Many avenues remain for future research.
Statistical characterizations of our fair PCA in asymptotic regime, as well as incorporation of sparsity \cite{JL09} are important open questions.
Incorporating stochastic optimization-type modifications \cite{Shamir15, fairbatch} is also an important direction, as such modifications are expected to result in better scalability and performance.

\section*{Acknowledgements}
\label{sec:acknowledgements}
JH, GS, and CD were partly supported by Institute for Information \& communications Technology Planning \& Evaluation(IITP) grant funded by the Korea government(MSIT) (No. 2019-0-01396, Development of framework for analyzing, detecting, mitigating of bias in AI model and training data), and partly supported by Institute for Information \& communications Technology Planning \& Evaluation(IITP) grant funded by the Korea government(MSIT) (No. 2021-0-01381, Development of Causal AI through Video Understanding).
JH was also supported by Institute of Information \& communications Technology Planning \& Evaluation (IITP) grant funded by the Korea government(MSIT)  (No. 2019-0-00075, Artificial Intelligence Graduate School Program(KAIST))

\noindent JH would like to thank Jaeho Lee and Se-Young Yun for providing helpful comments on the writing and organization of the paper.

\bibliography{fair-pca}

\clearpage
\title{SPs}
\appendix

\section{Relation to \cite{SamadiTMSV18}}
\label{sec:compare}
\begin{figure}[!t]
	\centering
	\includegraphics[width=\columnwidth]{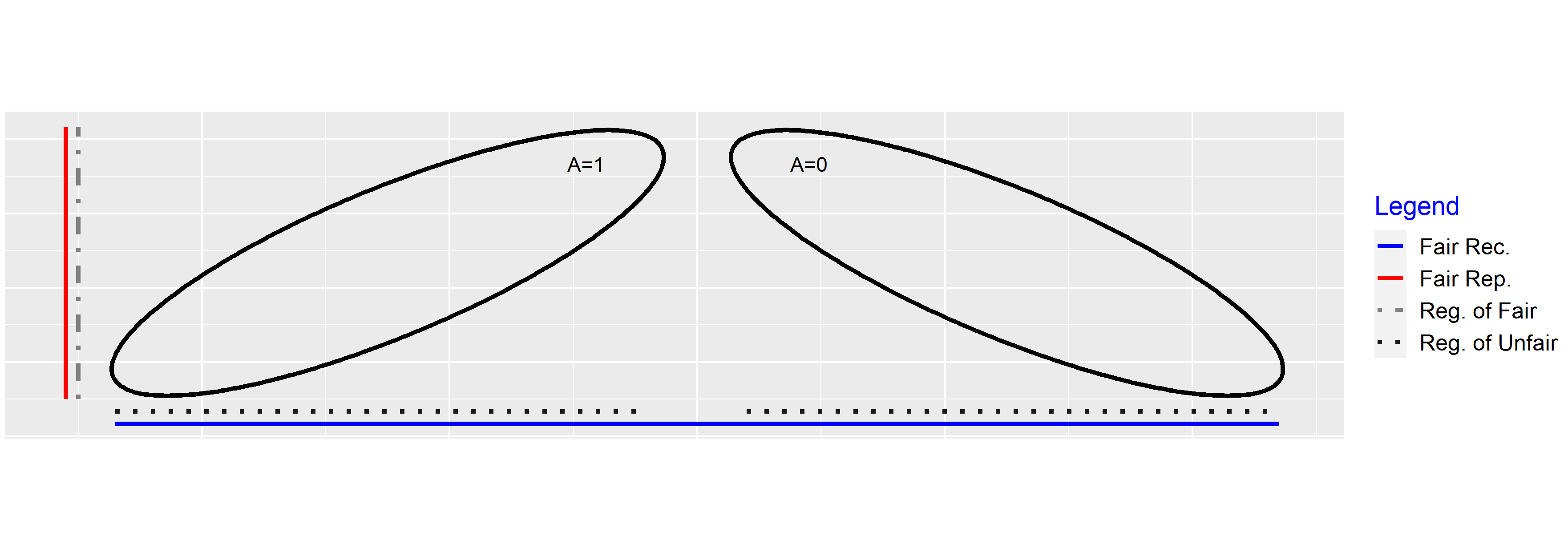}
	\caption{Comparison of fair PCA's as considered in \cite{SamadiTMSV18}, and our work and \cite{OA19}.}
\label{fig:c}
\end{figure}

There is another branch work \cite{SamadiTMSV18, TantipongpipatS19, KamaniHFFM19, PelegrinaBDRA20} in which the task of solving fair PCA is defined as finding a (linear) subspace in which the reconstruction errors for each protected class are approximately equal.
Here, we emphasize again that their fair PCA is {\it incomparable} with our considered fair PCA.

Consider the task of performing PCA on a two-dimensional data as shown in Figure~\ref{fig:c}, where $A \in \{ 0, 1\}$ denotes a protected attribute and the target dimension is set to $d = 1$.
We assume that the data distributions for both $A = 0$ and $A = 1$ take an elliptical density, the only difference being their orientation and position.

Fairness in reconstruction error can be achieved by projecting the data onto blue or red subspace.
However, fair PCA defined by \citeauthor{SamadiTMSV18} (\citeyear{SamadiTMSV18}) aims to {\it minimize} the reconstruction error with the fairness constraint, and thus the {\color{blue} blue} subspace is chosen.
Under our MMD-based definition of fair PCA, projecting onto the blue subspace is considered unfair because the distributions conditioned on each protected attribute $A$ are not the same; our fair PCA would choose the {\color{red} red} subspace.


One important observation is that while our definition of fairness ensures that the downstream tasks are fair w.r.t many of existing fairness definitions (see Proposition \ref{prop:fair-representation}), the reconstruction error-based definition has no such guarantees.

\section{Brief Description of FPCA}
\label{sec:fpca}

This section follows the discussion in \citeauthor{OA19} (\citeyear{OA19}), and by ``they", we refer to its authors.
First recall their proposed adversarial definition of fairness for PCA:
\testdef*

\subsection{Fairness Constraints}
By choosing $\mathcal{F}_c$ to be $\mathcal{F}_{lin} = \{h(u, t) = {\bf 1}(w^\intercal u - t \leq 0) : w \in \mathbb{R}^d\}$, $\Delta_A(\mathcal{F}_{lin})$ becomes the Kolmogorov distance between $X_s' \triangleq w^\intercal V^\intercal X | A = s$, which is upper bounded by their total variation distance.
Applying Pinsker's inequality \cite{Massart07}, we have that
\begin{equation*}
\Delta_A(\mathcal{F}_{lin}) \leq \sqrt{\frac{1}{2} \mathrm{KL}\left( X_0' || X_1' \right) }
\end{equation*}

In order to obtain a tractable form of fairness constraints, they make normality assumptions on the data as follows: $X_s = X | A = s \sim \mathcal{N}_p(\mu_s, \Sigma_s)$.
Then an upper bound of closed form can be obtained as follows:
\begin{eqnarray*}
& &     \Delta_A(\mathcal{F}_{lin})
\\
& \leq & \frac{1}{2} \sqrt{ \frac{|\Sigma_1|}{|\Sigma_0|} + \tr(\Sigma_1^{-1} \Sigma_0) - d + (\mu_1 - \mu_0)^\intercal \Sigma_1^{-1} (\mu_1 - \mu_0)},
\end{eqnarray*}
and thus, the sufficient condition for the RHS to be $0$ is that $\mu_0 = \mu_1$ and $\Sigma_0 = \Sigma_1$.

Denoting $f = \mu_1 - \mu_0$ and $Q = \Sigma_1 - \Sigma_0$, they introduced two constraints:
\begin{itemize}
\item Mean constraint: $h_m(V) := \lVert V^\intercal f \rVert = 0$

\item Covariance constraint: $h_c(V) = \lVert V^\intercal Q V \rVert_2 = 0$
\end{itemize}

Based on these, the optimization considered in their work is as follows:
\begin{equation}
\label{eq:FPCA}
\begin{aligned}
	& \text{minimize}
	& &  f(V) = - \langle \Sigma, V V^\intercal \rangle\\
	& \text{subject to}
	& & V^\intercal V = I_d,\\
	&&& h_m(V) = 0,\\
	&&& h_c(V) = 0.
\end{aligned}
\end{equation}

\subsection{SDP Formulation}
Standard convex relaxation techniques \cite{lmi-book} were then used to derive an SDP w.r.t. a {\it new} variable $P = V V^\intercal \in \mathbb{R}^{p \times p}$:
\begin{subequations}\label{eqn:sdpf}
\begin{align}
	\max\ &\langle X^{\textsf{T}}X,P\rangle - \mu t\\
	\text{s.t. } & \!\trace(P) \leq d\\
	&I_p\succeq P\succeq 0\\
	&\langle P, ff^\textsf{T}\rangle\leq \delta^2\label{eqn:mc}\\
	&\!\!\begin{bmatrix}t I_p & PM_+^{\textsf{T}} \\ PM_+ & I_p\end{bmatrix}\succeq 0\label{eqn:cc1}\\
	&\!\!\begin{bmatrix}tI_p & M_-^{\textsf{T}}P \\ PM_- & I_p\end{bmatrix}\succeq 0\label{eqn:cc2}
\end{align}
\end{subequations}
where $M_s M_s^\intercal$ is the Cholesky decomposition of $(-1)^s Q + \varphi I_p$, $\varphi \geq \lVert Q \rVert_2$, and $\delta, \mu \geq 0$ are the fairness tolerance levels\footnote{analogous to $\tau$ in our approach.} for mean and covariance constraints, respectively.
The final loading matrix $V$ is obtained by extracting the top $d$ eigenvectors from the resulting $P^*$.

\section{Low Explained Variance of FPCA}
\label{sec:low-sdp}
\begin{figure*}[!t]
\begin{center}
	\begin{subfigure}[t]{0.49\linewidth}
		\includegraphics[width=\linewidth]{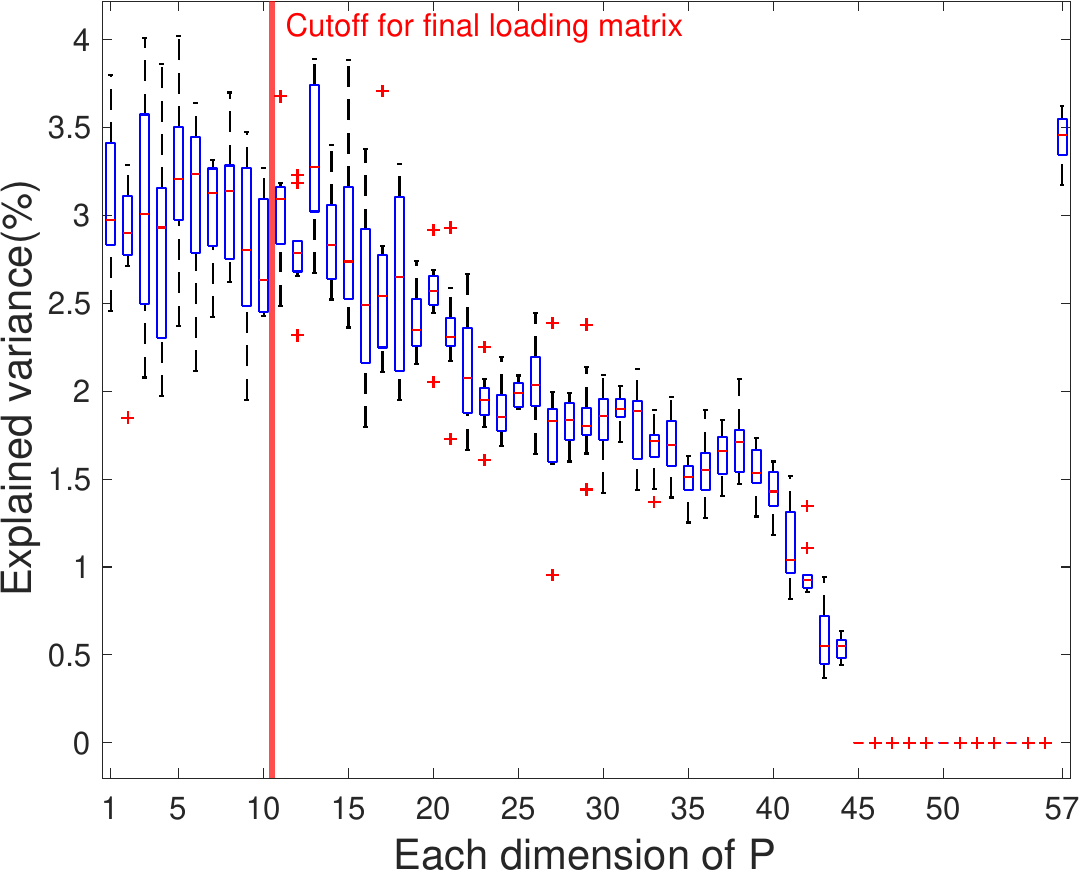}
		\caption{\label{fig:fpca-expvar} \textsc{FPCA} \cite{OA19}}
	\end{subfigure}\hfill
	\begin{subfigure}[t]{0.49\linewidth}
		\includegraphics[width=\linewidth]{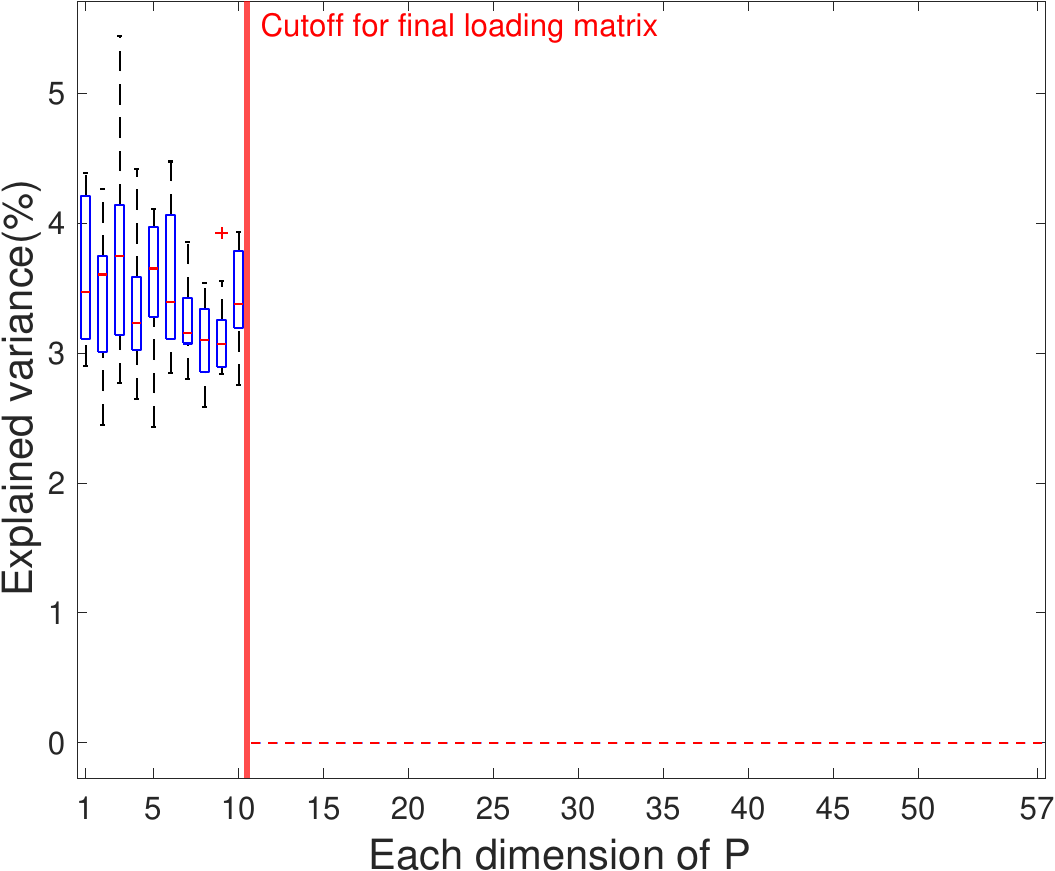}
		\caption{\label{fig:mbfpca-expvar} \textsc{MbF-PCA} (ours)}
	\end{subfigure}
	\caption{\label{fig:low-sdp} Explained variance of each eigenvector of $P^*$ for \textsc{German credit dataset}, over the considered $10$ train-test splits.
		Note how in \textsc{FPCA}'s case, the there's significant ``leakage" of explained variance in the latter part (i.e. starting from $11$-th eigenvector of $P$)}
\end{center}
\end{figure*}

\subsection{Empirical Exploration}

As seen in the Experiments, \textsc{FPCA} generally results in low explained variance than \textsc{MbF-PCA}, when {\it same(similar)} level of fairness is achieved.
For clarity of exposition, we consider \textsc{German credit dataset} with $d = 10$ and the following hyperparameter settings: $\tau = 10^{-3}$ for \textsc{MbF-PCA} and $(\delta, \mu) = (0, 0.1)$ for \textsc{FPCA}.
The reason for choosing a rather relaxed hyperparameter setting for \textsc{MbF-PCA} is to ensure a fair comparison i.e. to compare the two algorithms under the setting in which they achieve the {\it same} level of fairness.

To analyze this phenomenon in a detailed manner, we consider the variance explained of each eigenvector of $P^*$, {\it in order of the eigenvalues}.
Here, $P^*$ is the outputted $p \times p$ matrix from solving the SDP of \textsc{FPCA}.
Figure~\ref{fig:low-sdp} shows the plots of the explained variance of each eigenvector of $P^*$ for \textsc{FPCA} and \textsc{MbF-PCA}, over the considered $10$ train-test splits. 
(Since our algorithm directly outputs the loading matrix $V$, we simply set $P$ to be $V V^\intercal$)
Observe how the order of the eigenvalues {\bf do not match} the order of explained variance, and the discrepancy is very high.

\subsection{Theoretical Exploration}
We now give a sketch of theoretical argument of why such phenomenon of misalignment in ordering of eigenvalues of $P^*$ and explained variance of each eigenvector of $P^*$ may occur.
Again, we consider the Gaussian assumption of the data.

Let $\hat{\Sigma}$ be positive definite, and let $ \hat{\Sigma} = \sum_{k=1}^p \lambda_k w_k w_k^\intercal$ be its spectral decomposition with $\lambda_1 \ge \lambda_2 \ge \cdots \ge \lambda_p >0$ being the eigenvalues and $w_k$'s being the corresponding orthonormal eigenvectors.

Also, 
\begin{align}
\tr \big( V^\intercal \hat{\Sigma} V \big) &= \sum_{l=1}^d v_{l}^\intercal \left\{ \sum_{k=1}^p \lambda_k w_k w_k^\intercal  \right\} v_l \label{eqn:q1} \\
&= \sum_{k=1}^p \sum_{l=1}^d \lambda_k < v_l, w_k >^2 \\
&= \sum_{k=1}^p \sum_{l=1}^d \lambda_k \cos^2  \theta_{ v_l, w_k}
\end{align}

where $V^\intercal V = I_d$, $v_l$ is the $l$-th column of $V$, and $ \cos \theta_{u,v} := < u, v>.$
Eq.~\eqref{eqn:q1} shows that the maximization of $\tr \big( V^\intercal \hat{\Sigma} V \big)$ can be thought of as a optimization problem w.r.t. {\it orthonormal} $v_1, \dots, v_d$, {\it given} $\lambda_k$'s and $w_k$'s.
For example, with no constraint, the $\{ v_l^* \}_{l=1}^d$ such that $\cos \theta_{v_l^*, w_l } = 1$ achieve the maximization of $\tr \big( V^\intercal  \hat{\Sigma} V \big)$ and the variances explained are ordered by the $ \{ \lambda_l \}_{l=1}^d$, which is precisely Eckart-Young theorem \cite{pca}. 

When the covariance constraint $\Vert  V  Q V^\intercal \Vert_2 = 0$ is imposed on the
Eq.~\eqref{eqn:q1}, we will show the effect of this quadratic constraint to the optimal $V^*$.
The feasible space on the constraint  $\Vert  V^\intercal  Q V \Vert_2 = 0$ can be rewritten \cite{OA19}: for some fixed $\varphi \geq \lVert Q \rVert_2$, the covariance condition $\Vert V^\intercal Q V \Vert_2 = 0,$ is equivalent to
\begin{eqnarray}
\max \left\{ \Vert V^\intercal Q V + \varphi I_d  \Vert_2, \ \Vert V^\intercal Q V - \varphi I_d  \Vert_2 \right\} - \varphi = 0. \label{eqn:q2}
\end{eqnarray}
We can write $Q = \sum_{k=1}^p q_k s_k s_k^\intercal$, where $q_k$'s and $s_k$'s are {\it real} eigenvalues and {\it real} eigenvectors from the spectral decomposition of $Q$; note that all of them are real due to the symmetry of $Q$, but since $Q$ is not necessarily positive definite, $q_k$'s can be negative.
On the other hand, by construction, we have that $q_k + \varphi, -q_k + \varphi \geq 0 $ for all $k$'s.
Here, recall that for symmetric matrices, spectral norm is equivalent to the largest absolute value of eigenvalues.

Rewriting the LHS of Eq.~\eqref{eqn:q2}, we have that for either choice of $+$ and $-$,
\begin{align*}
	\left\lVert V^\intercal (Q \pm \varphi I_d) V \right\rVert_2 &= \left\lVert \sum_{k=1}^p (q_k \pm \varphi) V^\intercal s_k s_k^\intercal V \right\rVert_2 \\
	&= \max_{\lVert x \rVert = 1} \left| \sum_{k=1}^p (q_k \pm \varphi) x^\intercal V^\intercal s_k s_k^\intercal V x \right| \\
	&= \max_{\lVert x \rVert = 1} \left| \sum_{k=1}^p (\varphi \pm q_k) \left( x^\intercal V^\intercal s_k \right)^2 \right| \\
	&= \max_{\lVert x \rVert = 1} \sum_{k=1}^p (\varphi \pm q_k) \left( x^\intercal V^\intercal s_k \right)^2 \\
	&\leq \sum_{k=1}^p (\varphi \pm q_k) \max_{\lVert x \rVert = 1} \left( x^\intercal V^\intercal s_k \right)^2 \\
	&= \sum_{k=1}^p (\varphi \pm q_k) \left\lVert V^\intercal s_k \right\rVert^2 \\
	&= \sum_{k=1}^p (\varphi \pm q_k) \sum_{l=1}^d \cos^2 \theta_{v_l, s_k} \\
	&\leq \left( \max_k \left( \varphi \pm q_k \right) \right) \sum_{k=1}^p \sum_{l=1}^d \cos^2 \theta_{v_l, s_k}.
\end{align*}

Above and Eq. \eqref{eqn:q2} imply that
\begin{align*}
	\varphi &= \max \left\{ \Vert V^\intercal Q V + \varphi I_d  \Vert_2, \ \Vert V^\intercal Q V - \varphi I_d  \Vert_2 \right\} \\
	&\leq \max \left\{ \left( \max_k \left( \varphi + q_k \right) \right) \sum_{k=1}^p \sum_{l=1}^d \cos^2 \theta_{v_l, s_k}, \right. \\
	&\ \ \left. \left( \max_k \left( \varphi - q_k \right) \right) \sum_{k=1}^p \sum_{l=1}^d \cos^2 \theta_{v_l, s_k} \right\} \\
	&= \max\left\{ \varphi + q_1, \varphi - q_p \right\} \sum_{k=1}^p \sum_{l=1}^d \cos^2 \theta_{v_l, s_k} \\
	&= \left( \varphi + \max\left\{ q_1, - q_p \right\}\right) \sum_{k=1}^p \sum_{l=1}^d \cos^2 \theta_{v_l, s_k}
\end{align*}

If $\frac{\varphi}{\varphi + \max\{q_1, -q_p\}}$ is sufficiently large enough, then we would be constraining the sum of square of cosines to be somewhat close to $1$ i.e. forcing some of the to-be-found eigenvectors $v_l$'s into directions such that they are somewhat orthogonal to $s_k$'s, the eigenvectors of $Q$.
Such restriction becomes more prominent as $d$ increases.
All in all, the presence of the {\it covariance (quadratic)} constraint may cause a misalignment in the ordering of the explained variance of each loading vector, potentially causing the ordering of the explained variance to be substantially different from the ordering of the eigenvalues.


\section{Theoretical Minimum}
\label{sec:background}

In this section, we provide (minimum) required background for understanding the proofs of propositions/theorems presented in the main paper.
All proofs of the results here are deferred to the respective references and therein.

\subsection{A Primer on (Matrix) Manifolds, and Stiefel Manifold}

\subsubsection{Preliminaries}

(Here, we mainly follow the exposition from the textbook by \cite{AMS07})

Let $\mathcal{M}$ be a smooth, finite-dimensional Riemannian manifold, endowed with a Riemannian metric $\langle \cdot, \cdot \rangle_x$ on each tangent space $T_x\mathcal{M}$ and $x \in \mathcal{M}$.
Let $\lVert \cdot \rVert_x$ be the associated norm.
We omit $x$ when it is clear from context.
Let $\exp_x$ be the exponential map at $x$, and let $\mathrm{dist}$ be the Riemannian distance.






Let $f : \mathcal{M} \rightarrow \mathbb{R}$ be a smooth function.
We first consider the gradient of $f$, defined on manifolds::

\begin{definition}
The Riemannian gradient at $x \in \mathcal{M}$ of $f$, denoted as $\mathrm{grad} \ f(x)$, is the unique element of $T_x\mathcal{M}$ such that
\begin{equation*}
	\forall \xi \in T_x\mathcal{M} \quad \langle \mathrm{grad} \ f(x), \xi \rangle_x = D f(x) [\xi]
\end{equation*}

where $D f(x) [v]$ is the directional derivative of $f$ at $x$ along $v$.
\end{definition}

Finding a closed-form of $\mathrm{grad}$ may be difficult in some cases.
However, if $f$ can be extended to the ambient Euclidean space $\mathbb{R}^{p \times d}$, then $\mathrm{grad}$ can be computed as follows:
\begin{equation}
\mathrm{grad} \ f(x) = {\bf P}_{T_x\mathcal{M}}(\grad f(x))
\end{equation}

where $\grad$ is the usual Euclidean gradient, and ${\bf P}_{T_x\mathcal{M}}$ is the projection operator onto $T_x\mathcal{M}$.

\subsubsection{Matrix Manifolds}
Matrix is a mathematical object of fundamental importance.
Often in applications in machine learning and optimization, we are given the task of finding some matrix satisfying certain conditions, such as orthonnormality, while optimizing some real-valued function of the matrices \cite{HLWY20}.
In the perspective of optimization, it is natural to see whether an intrinsic geometric structure can be found in a set of matrices with certain common properties.
It turns out that such structures can be characterized by manifold, hence the term {\it matrix manifold}.

To start, let us consider $\mathbb{R}^{p \times d}$, a set of all possible real-valued $p \times d$ matrices.
Along with the Frobenius inner product, defined as $\langle X, Y \rangle := \tr(X^\intercal Y)$, $\mathbb{R}^{n \times p}$ can be regarded as a Euclidean space, and thus a Riemannian manifold.

Many of the well-known matrix manifolds are actually {\it embedded Riemannian submanifold} of $\mathbb{R}^{\it p \times d}$ i.e. its subspace topology is exactly the given topology on $\mathbb{R}^{\it p \times d}$.
We shall discuss about one specific matrix manifold that is crucial to our \textsc{MbF-PCA} in the next subsection.

\subsubsection{Stiefel Manifold}
We are mainly concerned with the {\it Stiefel manifold}, defined as follows\footnote{Another characterization of $St(p, d)$ is as a quotient space of set of $d$ linearly independent $p$-dimensional vectors.}:
\begin{equation}
St(p, d) = \{ V \in St(p, d) \ | \ V^\intercal V = I_d \}
\end{equation}

By considering the mapping $F(V) = V^\intercal V - I_d$, we can prove the following result:
\begin{theorem}
$St(p, d)$, defined as above, is an embedded (Riemannian) submanifold of $\mathbb{R}^{\it p \times d}$.
\end{theorem}

Now let us look at several properties of $St(p, d)$, many of which prove itself to be effective in both theoretical and practical perspective.

\begin{theorem}
\label{thm:st-connected}
$St(p, d)$ is $(p - d - 1)$-connected.
\end{theorem}
\begin{corollary}
When $p > d$, $St(p, d)$ is always at least path-connected, which implies connectedness.
\end{corollary}


\begin{theorem}
$St(p, d)$ is compact.
\end{theorem}

\begin{lemma}
Projection of any $X \in \mathbb{R}^{p \times d}$ onto $T_V St(p, d)$ is given as
\begin{equation}
	\mathbf{P}_{T_V St(p, d)} (X) = X - V \mathrm{sym}(V^\intercal X)
\end{equation}

where $\mathrm{sym}(X) = \frac{1}{2}(X + X^\intercal)$.
\end{lemma}

\section{Feasibility of Fair PCA}
\label{sec:feasibility}
Here, we briefly discuss the feasibility problem of fair PCA.

Let us consider Eq. \eqref{eq:FPCA}, which is a special case of our formulation of fair PCA under Gaussian assumption.
For our analysis, we consider Eq. \eqref{eq:FPCA} as a constrained manifold optimization over $St(p, d)$; geometric properties of manifolds turn out to be crucial in obtaining to-be-presented results.

When only the mean constraint is in place i.e. assume that $Q = O$, then the feasibility of Eq. \eqref{eq:FPCA} is always guaranteed:
\begin{proposition}
\label{prop:feasibility-1}
If $Q = O$, then Eq. \eqref{eq:FPCA} is always feasible.
\end{proposition}
\begin{proof}
We start by observing that the mean constraint, $V^\intercal f = 0$, is a linear hyperplane passing through the origin $0$, and the manifold constraint\footnote{In manifold optimization framework, this shouldn't be considered as an explicit constraint; we use such words here for clarity.}, $V^\intercal V = I_d$, is a manifold that is at least path-connected.
Now suppose that there exists some $f$ such that the linear hyperplane does not intersect $St(p, d)$.
Let $V \in St(p, d)$ be on the ``left" side of the hyperplane.
Since $St(p, d)$ is symmetric about the origin ($V^\intercal V = I_d \Rightarrow (-V)^\intercal (-V) = I_d$), $-V$ is on the ``right" side of the hyperplane.
Then there must exist some path on $St(p, d)$ connected $V$ and $-V$, which should intersect the hyperplane, a contradiction.
\end{proof}

When covariance constraint is also in-place, the problem becomes more complex.
However, especially in high dimensions, it is expected that the feasibility assumption is not so far-off.
As such discussion is not limited to Gaussian case of Eq. \eqref{eq:FPCA}, from hereon and forth we consider the following assumption:
\begin{assumption}
\label{assumption:feasible}
Eq. \eqref{eq:MbF-PCA} is feasible.
\end{assumption}

\section{Proof of Theoretical Guarantee of Algorithm \ref{alg:mbfpca}}
\label{sec:proofs}

\subsection{Assumptions}
Let us start by formally stating Assumption \ref{assumption:1}:
for each outputted solution of each Eq. \eqref{eq:subprob}, $V_{k+1}$, consider the following induced sequence $W_{k+1}$ defined as: 
\begin{equation}
W_{k+1} := \argmin_{\substack{V \in St(p, d) \\ \lVert \mathrm{grad} \mathcal{Q}(V, \rho_k) \rVert = 0}} \lVert V - V_{k + 1} \rVert.
\end{equation}
\begin{assumption}[formal; locality assumption]
\label{assumption:formal}
For each $k$, $W_{k+1}$ is a local minimum of Eq. \eqref{eq:subprob} i.e. for some fixed $r > 0$, any $W \in St(p, d) \cap B_{r}(W_{k+1})$ satisfies $\mathcal{Q}(W_{k+1}, \rho_{k+1}) \leq \mathcal{Q}(W, \rho_{k+1})$.
\end{assumption}

For simplicity, we omit the dependency of $W_{k+1}$ on $V_{k+1}$ and we write $k$ instead of $k + 1$.
Also, let us define
\begin{equation}
\label{eq:r-tilde}
\tilde{r} = \tilde{r}(\epsilon_{min}) := r - g(\epsilon_{min}).
\end{equation}

\subsection{Lemmas}
Let us first establish the existence of limit point of $\{V_k\}$:
\begin{lemma}
$\{V_k\}$ has at least one limit point.
\end{lemma}
\begin{proof}
Noting that for all $k$'s $V_k \in St(p, d)$, and that $St(p, d)$ is compact i.e. closed and bounded, by Bolzano-Weierstrass theorem, $\{V_k\}$ must have some limit point, which we denote as $\overline{V} \in St(p, d)$.
\end{proof}

Next lemma establishes an important connection between the two mentioned sequences, namely that there exists a limit point of $\{W_k\}$ that is very close (possibly equal) to $\overline{V}$:
\begin{lemma}
\label{lem:similar}
Let $\overline{V}$ be any limit point of $\{V_k\}$.
Then there exists some limit point $\overline{W}$ of $\{W_k\}$ and some decreasing, continuous function $g : \mathbb{R}_{\geq 0} \rightarrow \mathbb{R}_{\geq 0}$ satisfying $g(0) = 0$, such that
\begin{equation}
	\lVert \overline{V} - \overline{W} \rVert \leq g(\epsilon_{min}).
\end{equation}
\end{lemma}
\begin{proof}
Since $W_k \in St(p, d)$ for all $k$, there exists some subsequence $\{i_k\} \subset \{k\}$ such that $W_{i_k} \rightarrow \overline{W}$ for some $\overline{W} \in St(p, d)$, by Bolzano-Weierstrass theorem.
Obviously, $V_{i_k} \rightarrow \overline{V}$.
For simplicity, let us denote $V_{i_k}$ and $W_{i_k}$ as $V_k$ and $W_k$, respectively.

Then we denote   
\begin{equation}
	\left\{ V \in St(p, d) : \left\lVert \frac{1}{\rho_k} \mathrm{grad} f(V) + \mathrm{grad} h(V) \right\rVert \leq \frac{\epsilon}{\rho_k} \right\}
\end{equation}
by $G_k(\epsilon).$	As $\rho_k < \infty$, above is related to Algorithm \ref{alg:mbfpca} as 
\begin{equation}
	G_k(\epsilon) = \left\{ V \in St(p, d) : \lVert \mathrm{grad} \mathcal{Q}(V, \rho_k) \rVert \leq \epsilon \right\},
\end{equation}
where 
\begin{equation}
	\mathcal{Q}(V, \rho_k) = f(V) + \rho_k h(V).
\end{equation}

One crucial observation is that $G_k(\epsilon) \subseteq G_k(\epsilon')$ for $\epsilon \leq \epsilon'$.

Let $\varepsilon > 0$ be arbitrary.
By given, there exists some $K > 0$ such that for any $k > K$, $\lVert V_k - \overline{V} \rVert < \varepsilon / 2$ and $\lVert W_k - \overline{W} \rVert < \varepsilon / 2$.
By triangle inequality,
\begin{align*}
	\lVert \overline{V} - \overline{W} \rVert &\leq \lVert V_k - \overline{V} \rVert + \lVert W_k - \overline{W} \rVert + \lVert V_k - W_k \rVert \\
	&< \varepsilon + \lVert V_k - W_k \rVert.
\end{align*}

We can further bound the second term as follows:
\begin{align*}
	\lVert V_k - W_k \rVert &= \min_{W \in G_k(0)} \lVert V_k - W \rVert \\
	&\leq \max_{V \in G_k(\epsilon_k)} \min_{W \in G_k(0)} \lVert V - W \rVert
\end{align*}
where the first equality follows from the definition of $\{W_k\}$, and the last inequality follows from the fact that our algorithm satisfies $V_k \in G_k(\epsilon_k)$.

Note that since $G_k(0) \subseteq G_k(\epsilon_k)$, above is the variational form of the Hausdorff distance \cite{rockafellar-variational} between $G_k(0)$ and $G_k(\epsilon_k)$, which we denote as $d_H(\cdot, \cdot)$. Additionally, we denote 
\begin{equation}
	\left\{ V \in St(p, d) : \left\lVert \frac{1}{\rho} \mathrm{grad} f(V) + \mathrm{grad} h(V) \right\rVert \leq \frac{\epsilon}{\rho} \right\}
\end{equation}
by $G_\infty(\epsilon)$, 
where $\rho := \lim_{k \rightarrow \infty} \rho_k \leq \rho_{max} \leq \infty$ and $\frac{1}{\infty} = 0$.

Then it is easy to see that as $k \rightarrow \infty$, $G_k(0) \rightarrow G_\infty(0)$, $G_k(\epsilon_k) \rightarrow G_\infty(\epsilon_{min})$, and $G_\infty(0) \subset G_\infty(\epsilon_{min})$.
Now define
\begin{equation}
	g(\epsilon) := d_H(G_\infty(0), G_\infty(\epsilon)).
\end{equation}

Then $g$ is continuous and  decreasing with respect to $\epsilon$ since $G_\infty(\epsilon) \subseteq G_\infty(\epsilon')$ for $\epsilon \leq \epsilon'$, and it satisfies $g(0) = 0$.

Since $\varepsilon > 0$ was arbitrary, the statement follows.
\end{proof}

Now let us prove the Lipschitzness of the functions considered in our fair PCA:
\begin{lemma}
$f$ is $2\lVert \Sigma \rVert$-Lipschitz.
\end{lemma}
\begin{proof}
For all $V, W \in St(p, d)$,
\begin{eqnarray*}
	& &     \lVert V^\intercal \Sigma V - W^\intercal \Sigma W \rVert 
	\\ 
	&\leq& \lVert V^\intercal \Sigma (V - W) \rVert + \lVert (V - W)^\intercal \Sigma W \rVert \\
	&\leq& \lVert V \rVert \lVert \Sigma \rVert \lVert V - W \rVert + \lVert V - W \rVert \lVert \Sigma \rVert \lVert W \rVert \\
	&\leq& 2 \lVert \Sigma \rVert \lVert V - W \rVert.
\end{eqnarray*}
\end{proof}

\begin{lemma}
There exists some $L_h > 0$ such that $h : St(p, d) \rightarrow \mathbb{R}$ is $L_h$-Lipschitz.
\end{lemma}
\begin{proof}
Suppose not.
Then for any $L_h > 0$, there exists some $V, W \in St(p, d)$ such that $|h(V) - h(W)| > L_h \lVert V - W \rVert$.
Especially for $W \not= V$, we have that $L_h < \frac{|h(V) - h(w)|}{\lVert V - W \rVert}$.
Letting $W \rightarrow V$, we have that $L_h < |\grad_V h(V)|$.
This implies that $\grad_x h$ is unbounded on $St(p, d)$, a contradiction since the continuity of $\grad_V h$ and compactness of $St(p, d)$ implies that $\grad_V h$ is bounded by the Extreme Value theorem.
\end{proof}

As done in proof of Lemma \ref{lem:similar}, let us define $\rho \triangleq \lim_{k \rightarrow \infty} \rho_k$.
We now show how $\rho_{max}, \tau$ affect $\rho$:
\begin{lemma}
\label{lem:rho}
$\rho = \rho(\rho_{max}, \tau)$ is characterized as follows: there exists some function $K := K(\tau) \in \mathbb{N}$, satisfying $K < \infty$ for $\tau > h(\overline{V})$ and $K \leq \infty$ otherwise, such that
\begin{equation}
	\label{eq:rho}
	\rho = \min(\theta_\rho^K \rho_0, \rho_{max}).
\end{equation}

Moreover, $K$ is decreasing in $\tau$ and increasing in $\rho_{max}$.
\end{lemma}
\begin{proof}
Define $K(\tau) := \left| \{ k \in \mathbb{N}: h(V_k) > \tau \} \right|$, which is the number of $\rho$-updates in the algorithm.
From the algorithm, Eq. \eqref{eq:rho} is directly implied.
Now it remains to characterize the function $K(\tau)$.

Let us first consider the case of $\tau > h(\overline{V})$.
Since $h(V_k) \rightarrow h(\overline{V})$ and $h$ is continuous, it must be that $K(\tau) < \infty$ by the algorithm, regardless of the other hyperparameters.
Now suppose that $\tau \leq h(\overline{V})$.
Then there are two cases to consider:
\begin{enumerate}
	\item If $h(V_M) \leq \tau$ for some $M < \infty$, then from the algorithm it is clear that $K(\tau) \leq M$.
	
	\item If $h(V_k) > \tau$ for all $k$, then again, from the algorithm it is clear that $K(\tau) = \infty$.
\end{enumerate}

By construction, $K(\tau)$ is decreasing in $\tau$, and thus, $\rho$ is decreasing in $\tau$ and increasing in $\rho_{max}$.
\end{proof}

\begin{corollary}
If no $k$ satisfies $h(V_k) \leq \tau$, then $\rho = \infty$ if and only if  $\rho_{max} = \infty, \tau = 0$.
\end{corollary}

Denote $\mathcal{C} \triangleq St(p, d) \cap h^{-1}(0)$ as the set of feasible points.
By compactness of $St(p, d)$ and continuity of $h$, we have that $\mathcal{C}$ is also compact.
By Assumption \ref{assumption:feasible}, $\mathcal{C}$ is non-empty.
Lastly, recall from Eq. \eqref{eq:r-tilde} $\tilde{r} = r - g(\epsilon_{min})$.
The last lemma asserts that after sufficiently many iterations have passed, the sequences $\{V_k\}$ and $\{W_k\}$ are always close to $\mathcal{C}$.
\begin{lemma}
\label{lem:exist}
If $\overline{V}$ is fair and $\epsilon_{min}$ is sufficiently small in the sense that $g(\epsilon_{min}) < r$, then there exists some $K > 0$ such that for all $k > K$, $B_{\tilde{r}}(V_k) \cap \mathcal{C} \not= \emptyset$ and $B_r(W_k) \cap \mathcal{C} \not= \emptyset$.
\end{lemma}
\begin{proof}
First, suppose that for all $k$, $B_{\tilde{r}}(V_k) \cap \mathcal{C} = \emptyset$.
Taking the limit of $k \rightarrow \infty$ on both sides, we have that $\overline{V} \in B_{\tilde{r}}(\overline{V}) \cap \mathcal{C} = \emptyset$ by the given assumption, a contradiction.
Thus there exists some $K_1$ such that for all $k > K_1$, $B_{\tilde{r}}(V_k) \cap \mathcal{C} \not= \emptyset$.

Now suppose that for all $k$, $B_r(W_k) \cap \mathcal{C} = \emptyset$.
Recall from Lemma \ref{lem:similar} that
\begin{equation*}
	\lVert \overline{W} - \overline{V} \rVert \leq g(\epsilon_{min}) < r
\end{equation*}

Again taking the limit of $k \rightarrow \infty$, we have that $\overline{V} \in \mathcal{C} \cap B_r(\overline{W}) = \emptyset$, a contradiction.
Thus there exists some $K_2$ such that for all $k > K_2$, $B_r(W_k) \cap \mathcal{C} \not= \emptyset$.

Set $K = \max(K_1, K_2)$, and we are done.
\end{proof}

\subsection{Proof of Theorem \ref{thm:mbfpca-another}}
(The proof is inspired from the optimality guarantee for the Euclidean quadratic penalty method; see Theorem 17.1 of \citeauthor{nocedal-wright}, \citeyear{nocedal-wright})
The proof is divided into 3 parts.

\textit{ {\textbf {Existence of limit point of $\{W_k\}$}}}\\
By Lemma \ref{lem:similar}, there exists some limit point $\overline{W}$ of $\{W_k\}$ such that $\lVert \overline{V} - \overline{W} \rVert \leq g(\epsilon_{min})$.
By considering appropriate subsequence, let $W_k \rightarrow \overline{W}$.

\textit{ {\textbf {$\overline{V}$ is approximate local minimizer of \eqref{eq:auxiliary}}}}\\
By Assumption \ref{assumption:formal}, we have that for all $W \in B_r(W_k) \cap St(p, d)$,
\begin{equation}
\mathcal{Q}(W_k, \rho_k) \leq \mathcal{Q}(W, \rho_k)
\end{equation}

i.e.
\begin{equation}
\label{eq:local}
f(W_k) + \rho_k h(W_k) \leq f(W) + \rho_k h(W).
\end{equation}

Rearranging gives
\begin{equation}
h(W_k) \leq h(W) + \frac{1}{\rho_k} \left( f(W) - f(W_k) \right)
\end{equation}

Taking the limit $k \rightarrow \infty$ gives
\begin{align*}
h(\overline{W}) &\leq h(W) + \frac{1}{\rho} \left( f(W) - f(\overline{W}) \right) \\
&\leq h(W) + \frac{2 \lVert \Sigma \rVert}{\rho} \lVert W - \overline{W} \rVert,
\end{align*}

for all $W \in B_r(\overline{W}) \cap St(p, d)$.

By Lemma \ref{lem:similar}, $B_{\tilde{r}}(\overline{V}) \cap St(p, d) \subset B_r(\overline{W}) \cap St(p, d)$.
Thus we have that for all $V \in B_{\tilde{r}}(\overline{V}) \cap St(p, d)$,
\begin{align*}
h(\overline{V}) &\leq h(\overline{W}) + |h(\overline{V}) - h(\overline{W})| \\
&= h(V) + \left( h(\overline{W}) - h(V) \right) + L_h \lVert \overline{V} - \overline{W} \rVert \\
&\leq h(V) + \frac{2 \lVert \Sigma \rVert}{\rho} \lVert V - \overline{W} \rVert + L_h \lVert \overline{V} - \overline{W} \rVert \\
&\leq h(V) + \frac{2 \lVert \Sigma \rVert}{\rho} \left( \lVert V - \overline{V} \rVert + \lVert \overline{V} - \overline{W} \rVert \right) + L_h \lVert \overline{V} - \overline{W} \rVert \\
&= h(V) + \frac{2 \lVert \Sigma \rVert}{\rho} \lVert V - \overline{V} \rVert + \left( L_h + \frac{2 \lVert \Sigma \rVert}{\rho} \right) g(\epsilon_{min}).
\end{align*}

Define $\beta(\rho_{max}, \tau) := \frac{2 \lVert \Sigma \rVert}{\rho(\rho_{max}, \tau)}$.
By Lemma \ref{lem:rho}, $\beta$ indeed satisfies $\lim_{\rho_{max} \rightarrow \infty} \rho(\rho_{max}, 0) \geq 0$ and $\lim_{\rho_{max} \rightarrow \infty} \rho(\rho_{max}, \tau) > 0$ for $\tau > 0$.
Also, by the same lemma, $\beta$ is indeed decreasing in $\tau$ and increasing in $\rho_{max}$.
(Note how this particular part of the proof does {\it not} rely on the assumption that $\overline{V}$ is feasible!)

\textit{ {\textbf {If $\overline{V}$ is fair, then $\overline{V}$ is an approximate local minimizer of Eq. \eqref{eq:MbF-PCA}}}}\\
By Lemma \ref{lem:exist}, we may only consider $k$'s such that $B_{\tilde{r}}(V_k) \cap \mathcal{C} \not= \emptyset$ and $B_r(W_k) \cap \mathcal{C} \not= \emptyset$.
By Assumption \ref{assumption:formal}, we have that for all $W \in B_r(W_k) \cap \mathcal{C}$,
\begin{equation}
\mathcal{Q}(W_k, \rho_k) \leq \mathcal{Q}(W, \rho_k)
\end{equation}

i.e.
\begin{equation}
\label{eq:local}
f(W_k) \leq f(W_k) + \rho_k h(W_k) \leq f(W) + \rho_k h(W) = f(W),
\end{equation}
where the last equality follows from $W \in \mathcal{C}$.

Thus we have that
\begin{align*}
f(\overline{V}) &\leq f(\overline{W}) + |f(\overline{V}) - f(\overline{W})| \\
&\leq \lim_{k \rightarrow \infty} f(W_k) + 2\lVert \Sigma \rVert \lVert \overline{V} - \overline{W} \rVert \\
&\leq f(W) + 2\lVert \Sigma \rVert g(\epsilon_{min}),
\end{align*}

for all $W \in B_r(W_k) \cap \mathcal{C}$.
Specifically, since $B_{\tilde{r}}(\overline{V}) \cap St(p, d) \subset B_r(\overline{W}) \cap St(p, d)$ by Lemma \ref{lem:similar}, we have that

\begin{equation*}
f(\overline{V}) \leq f(V) + 2\lVert \Sigma \rVert g(\epsilon_{min}),
\end{equation*}

for all $V \in B_{\tilde{r}}(\overline{V}) \cap \mathcal{C}$.

\subsection{Proof of Theorem \ref{thm:mbfpca}}
Set $\epsilon_{min} = 0$, $\tau = 0$, and $\rho_{max} = \infty$ in the above proof.
(Note that $g(0) = 0$ implies that $\overline{V} = \overline{W}$)
Just as in proof of Lemma \ref{lem:rho}, if there exists some $M < \infty$ such that $h(V_M) = 0$, then from the algorithm it is clear that $\overline{V}$ is fair and we are done.
If not, then we must have that $\rho(\infty, 0) = \infty$ i.e. $\beta(\infty, 0) = 0$ i.e. $\overline{V}$ is a local minimizer of Eq. \ref{eq:auxiliary}.

\section{Closed form of $\grad_V h(V)$}
\label{sec:closed-form}
In this section, we derive the closed form of the constraint gradient.

First, we recall the lemma that provides a closed form of $h(V) = \MMD^2(\cdot, \cdot)$:
\testlem*

For simplicity, we use the followin abbreviations: $\sum_{i,j=1}^m := \sum_{i=1}^m \sum_{j=1}^m$ and $\sum_{i,j=1}^{m,n} := \sum_{i=1}^m \sum_{j=1}^n$.

After squaring both sides, let us refer to the three terms as $h_1, h_2$, and $h_3$ i.e. $h(V) = \widehat{\Delta}^2 = h_1(V) + h_2(V) - 2h_3(V)$ where
\begin{equation*}
h_1(V) = \frac{1}{m^2} \sum_{i,j = 1}^m k(VX_i, VX_j),
\end{equation*}
\begin{equation*}
h_2(V) = \frac{1}{n^2} \sum_{i,j = 1}^n k(VY_i, VY_j),
\end{equation*}
\begin{equation*}
h_3(V) = \frac{1}{mn} \sum_{i, j = 1}^{m, n} k(VX_i, VY_j).
\end{equation*}

\subsection{Closed form of $\grad_V h_1(V)$}
Denote 
\begin{eqnarray*}
g_{ij}(V) &=& \tr((X_i - X_j)^\intercal V V^\intercal (X_i - X_j)) 
\\
K_{ij}(V) &=& \exp(-g_{ij}(V) / (2 \sigma^2))
\\
K_X(V) &=& \{K_{ij}(V)\} \in \mathbb{R}^{m \times m},
\end{eqnarray*} 
and $X$ be the data matrix whose $i$-th row is $X_i^\intercal$.
By chain rule, we have that
\begin{align*}
\grad_V h_1(V) &= -\frac{1}{2 m^2 \sigma^2} \sum_{i, j = 1}^m K_{ij}(V) \grad_V g_{ij}(V) \\
&= -\frac{1}{m^2 \sigma^2} \sum_{i, j = 1}^m K_{ij}(V)  (X_i - X_j) (X_i - X_j)^\intercal V \\
\end{align*}

Now denote $H_i(V) = \sum_{j = 1}^m K_{ij}(V)$ and $H_X(V) = \mathsf{diag}(H_1(V), H_2(V), \cdots, H_m(V))$, where `$\mathsf{diag}$' is the block-diagonal operator.
We have that
\begin{align*}
\sum_{i, j = 1}^m K_{ij}(V) X_i X_i^\intercal &= \sum_{i=1}^m \left( \sum_{j=1}^m K_{ij}(V) \right) X_i X_i^\intercal \\
&= X^\intercal H_X(V) X,
\end{align*}

and
\begin{equation*}
\sum_{i, j = 1}^m K_{ij}(V) X_i X_j^\intercal = X^\intercal K_X(V) X.
\end{equation*}

Since $K_X(V)$ is symmetric i.e. $K_{ij}(V) = K_{ji}(V)$,
\begin{equation*}
\sum_{i, j = 1}^m K_{ij}(V) X_j X_j^\intercal = 
\sum_{i, j = 1}^m K_{ji}(V) X_j X_j^\intercal =
X^\intercal H_X(V) X,
\end{equation*}

and
\begin{equation*}
\sum_{i, j = 1}^m K_{ij}(V) X_j X_i^\intercal = 
\sum_{i, j = 1}^m K_{ji}(V) X_j X_i^\intercal =
X^\intercal K_X(V) X.
\end{equation*}

Thus we have that
\begin{equation*}
\grad_V h_1(V) = -\frac{2}{m^2 \sigma^2} X^\intercal (H_X(V) - K_X(V)) X V.
\end{equation*}

\subsection{Closed form of $\grad_V h_2(V)$}
The computation is almost exactly the same as above, and thus we only show the final result:
\begin{align*}
\grad_V h_2(V) = -\frac{2}{n^2 \sigma^2} Y^\intercal (H_Y(V) - K_Y(V)) Y V,
\end{align*}
where $Y$, $H_Y(V)$, and $K_Y(V)$ are defined analogously.

\subsection{Closed form of $\grad_V h_3(V)$}
Let us overload the notation a bit and now denote
\begin{eqnarray*}
g_{ij}(V) &=& \tr((X_i - Y_j)^\intercal V V^\intercal (X_i - Y_j)),
\\
K_{ij}(V) &=& \exp(-g_{ij}(V) / (2 \sigma^2)),
\\
K_{XY}(V) &=& \{K_{ij}(V)\} \in \mathbb{R}^{m \times n}.
\end{eqnarray*} 

Again by chain rule, we have that
\begin{align*}
\grad_V h_3(V) &= -\frac{1}{2 mn \sigma^2} \sum_{i, j = 1}^{m, n} K_{ij}(V) \grad_V g_{ij}(V) \\
&= -\frac{1}{mm \sigma^2} \sum_{i, j = 1}^{m, n} K_{ij}(V)  (X_i - Y_j) (X_i - Y_j)^\intercal V. 
\end{align*}

Now, again with a slight notation overload, denote 
\begin{eqnarray*}
H_i(V) &=& \sum_{j = 1}^n K_{ij}(V), \\
H_{XY}(V) &=& \mathrm{diag}(H_1(V), H_2(V), \cdots, H_m(V)), \\
\tilde{H}_j(V) &=& \sum_{i = 1}^m K_{ij}(V), \\
\tilde{H}_{XY}(V) &=& \mathrm{diag}(\tilde{H}_1(V), \tilde{H}_2(V), \cdots, \tilde{H}_n(V)).
\end{eqnarray*}

We have that
\begin{align*}
\sum_{i, j = 1}^{m, n} K_{ij}(V) X_i X_i^\intercal &= \sum_{i=1}^m \left( \sum_{j=1}^n K_{ij}(V) \right) X_i X_i^\intercal \\
&= X^\intercal H_{XY}(V) X,
\end{align*}

and
\begin{align*}
\sum_{i, j = 1}^{m, n} K_{ij}(V) Y_j Y_j^\intercal &= \sum_{j=1}^n \left( \sum_{i=1}^m K_{ij}(V) \right) Y_j Y_j^\intercal \\
&= Y^\intercal \tilde{H}_{XY}(V) Y.
\end{align*}

Next, we have that
\begin{equation*}
\sum_{i, j = 1}^{m, n} K_{ij}(V) X_i Y_j^\intercal = X^\intercal K_{XY}(V) Y,
\end{equation*}

and
\begin{equation*}
\sum_{i, j = 1}^{m, n} K_{ij}(V) Y_j X_i^\intercal = Y^\intercal K_{XY}(V)^\intercal X.
\end{equation*}

Thus we have that
\begin{align*}
\grad_V h_3(V) = &-\frac{1}{mn \sigma^2} \Big( X^\intercal H_{XY}(V) X + Y^\intercal \tilde{H}_{XY}(V) Y \\
&- X^\intercal K_{XY}(V) Y - Y^\intercal K_{XY}(V)^\intercal X \Big) V.
\end{align*}








\section{Description of 2nd Part of Experiment 1}
\label{sec:exp1-description}
With $p_0 = 1000$ and $n = 250$, we first constructed two {\it different} $p_0$-variate Gaussian distributions $\mathcal{N}_{p_0}(\mu_{p_0}^{(s)}, \Sigma_{p_0}^{(s)})$ and sampled $n$ points for each $s \in \{0, 1\}$.
Then for each considered $p$, we projected the samples by multiplying the constructed data matrices with a Gaussian random matrix of size $p_0 \times p$.
This creates $18$ pairs of $p$-dimensional distributions.
This allowed for us model the situation in which two protected groups have different {\it non-Gaussian} distributions.

We now describe how we've set $\mu_{p_0}^{(s)}$ and $\Sigma_{p_0}^{(s)}$ for each $s$.
The mean difference is set as $f_{p_0} = \mu_{p_0}^{(1)} - \mu_{p_0}^{(0)}$ where $2 * f^{raw}_{p_0} / \lVert f^{raw}_{p_0} \rVert$ and $f^{raw}_{p_0} = \bf{1}$.
Then $\mu_{p_0}^{(0)}$ is set as $\bf{0}$ and $\mu_{p_0}^{(1)}$ is set as $\mu_{p_0}^{(0)} + f_{p_0}$.

The covariance difference is set as $Q^{raw}_p /  \lVert Q^{raw}_p \rVert_2$ where
\begin{equation*}
Q^{raw}_p = A_p^{(1)} - A_p^{(0)},
\end{equation*}
\begin{align*}
\Sigma_p^{(0)} = \mathsf{diag} \big(AR_{p/5}(0.99), AR_{p/5}(0.98), AR_{p/5}(0.97), \\ AR_{p/5}(0.98), AR_{p/5}(0.95) \big),
\end{align*}
and
\begin{align*}
\Sigma_p^{(1)} = \mathsf{diag} \big(AR_{p/5}(0.99), AR_{p/5}(0.98), AR_{p/5}(0.97), \\ AR_{p/5}(0.98), AR_{p/5}(0.99) \big).
\end{align*}
Here, `$\mathsf{diag}$' is the block-diagonal operator, and
\begin{equation}
\label{eq:ar1}
\left(AR_{p/5}(r)\right)_{ij} = \frac{r^{|i - j|}}{1 - r^2},
\ \
i, j \in \{ 1, \ldots, p/5\}
\end{equation}
is the covariance matrix of a $p/5$-variate Gaussian AR(1) process \cite{Shumway-Stoffer}.
The $r \in (0, 1)$ is a parameter that controls the strength of correlation.
Finally, $\Sigma_p^{(0)}$ is set as $A_p^{(s)} / \lVert Q_p \rVert$.

Due to the scaling, $\lVert f_p \rVert = 2$ and $\lVert Q_p \rVert = 1$ for all $p$.
This is to ensure that $\lVert f_p \rVert$ and $\Vert Q_p \rVert$ do not explode as $p$ increases, allowing for us to see the effect of high dimensions more clearly under the ``same" fairness setting.

\section{Description of Experiment 2}
\label{sec:exp2-description}
Here, we describe the full pre-processing steps for the three UCI datasets considered in this paper.
All pre-processing steps were done using AI Fairness 360 (AIF360; \citeauthor{aif360}, \citeyear{aif360}) built-in functionalities.
In our Github repository\footnote{\url{https://github.com/nick-jhlee/fair-manifold-pca}}.
, we include the links for downloading the raw datasets, along with the pre-processing codes.

\subsection{COMPAS dataset}
The raw data was randomly reduced to 40\%, resulting in $2468$ samples.
This was to avoid computational burden in computing the $\MMD^2$ metric, as it scales quadratically in number of samples.
'Race' was set as the protected attribute, and the task is to classify the bank account holders into credit class {\it Good} or {\it Bad}.
The features `sex' and `c\_charge\_desc' were dropped, resulting in total of $11$ features.

\subsection{German credit dataset}
The total number of samples was $1000$.
`Age' was set as the protected attribute as done in \cite{KamiranC11}, and the task is to classify whether the individual bank account holder has good or bad credit class.
The features `sex` and `personal\_status` were dropped, resulting in total of $57$ features.

\subsection{Adult income dataset}
The raw data was randomly reduced to 5\%, resulting in $2261$ samples; this was due to the same reason as the COMPAS dataset.
`Sex' was set as the protected attribute as done in \cite{Kamishima11}, and the task is to indicate whether the individual's income is larger than 50K dollars.
The features `fnlwgt` and `race` were dropped, resulting in total of $99$ features.

\onecolumn
\section{Full results for EDA of considered UCI datasets}
\label{sec:eda-full}
We show two things:
\begin{itemize}
    \item Boxplots of communality of the protected attribute of each considered UCI dataset, for both $d = 2$ and $d = 10$. (Figures \ref{fig:compas-box}, \ref{fig:german-box}, \ref{fig:adult-box})
    
    \item PC ($d = 2$) biplot of each considered UCI dataset, for all considered algorithms. (Figures \ref{fig:compas-biplot-pca}, \ref{fig:compas-biplot-fpca}, \ref{fig:compas-biplot-mbfpca}, \ref{fig:german-biplot-pca}, \ref{fig:german-biplot-fpca}, \ref{fig:german-biplot-mbfpca}, \ref{fig:adult-biplot-pca}, \ref{fig:adult-biplot-fpca}, \ref{fig:adult-biplot-mbfpca})
\end{itemize}

Particularly for biplots, one train-test split was chosen, and based on that, only the top $10$ features with the highest communalities (+ the sensitive attribute, if it is not included) were chosen for plots.


\begin{figure*}[!h]
	\begin{center}
		\begin{subfigure}[t]{0.49\linewidth}
				\includegraphics[width=\linewidth]{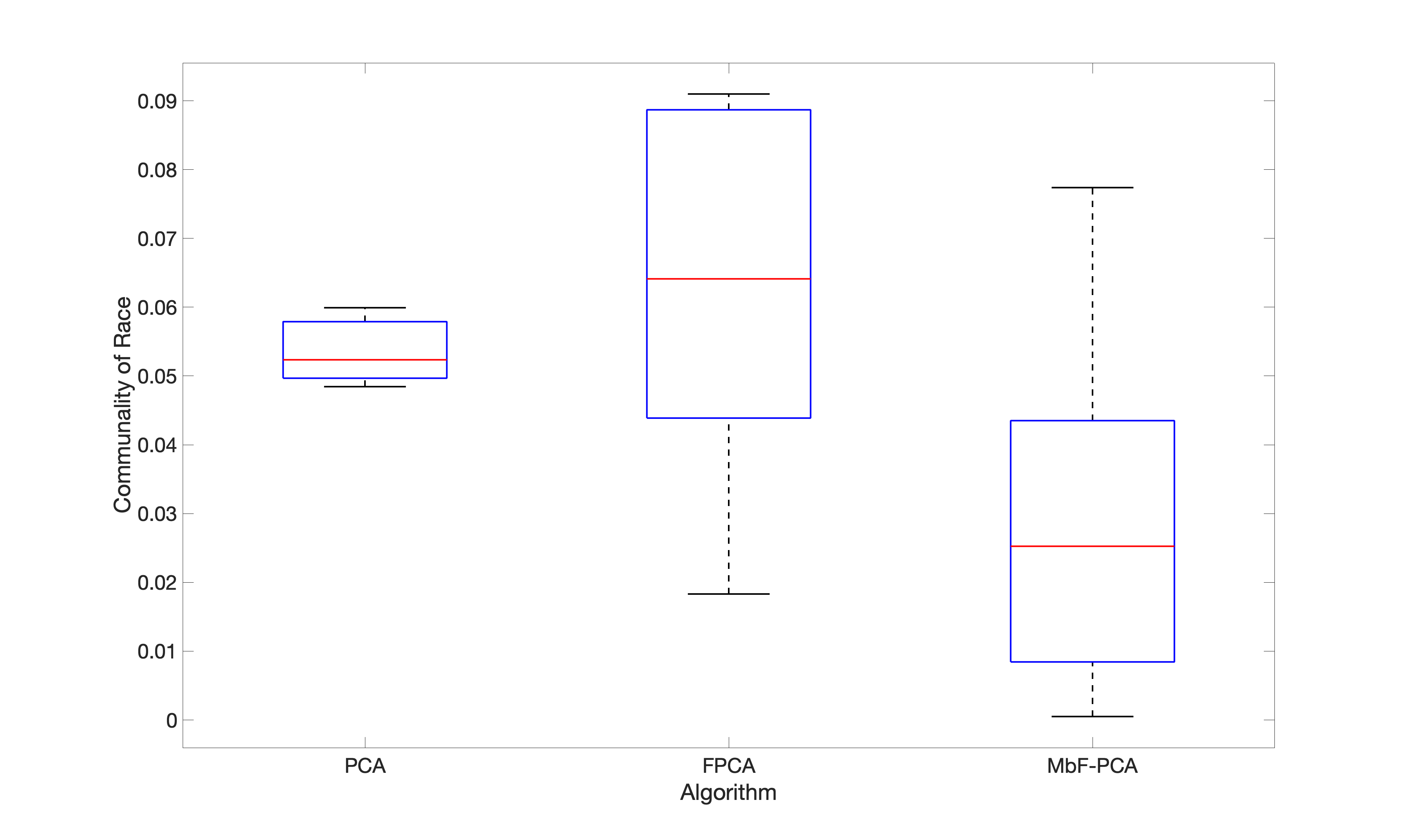}
				\caption{\label{fig:compas-2-box} $d = 2$}
			\end{subfigure}\hfill
		\begin{subfigure}[t]{0.49\linewidth}
				\includegraphics[width=\linewidth]{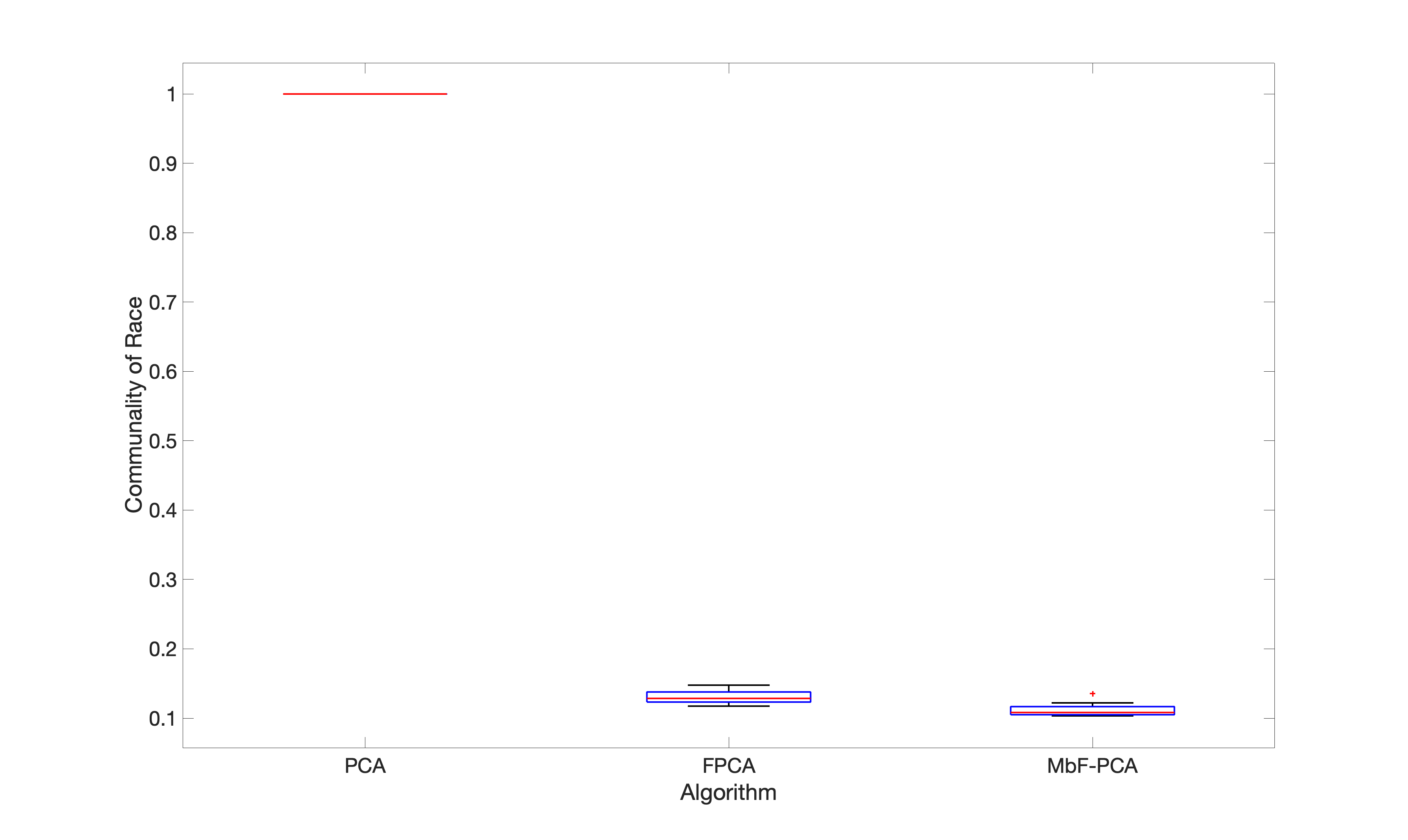}
				\caption{\label{fig:compas-10-box} $d = 10$}
			\end{subfigure}
		\caption{\label{fig:compas-box} Boxplots of communalities of race for COMPAS dataset}
	\end{center}
\end{figure*}
\begin{figure*}[!h]
    \centering
    \includegraphics[width=0.9\linewidth]{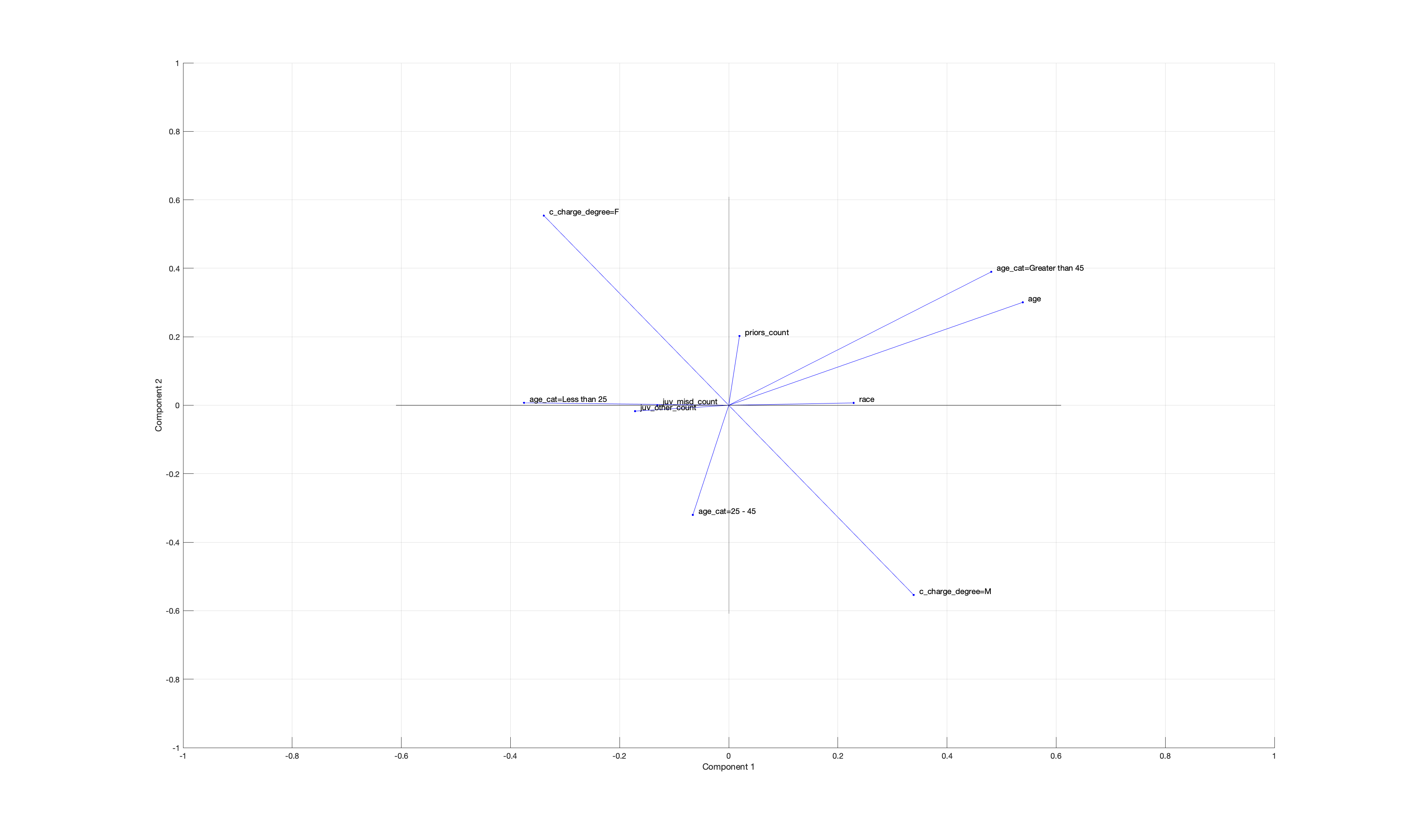}
    \caption{PC biplot of PCA for COMPAS dataset}
    \label{fig:compas-biplot-pca}
\end{figure*}
\begin{figure*}[!h]
    \centering
    \includegraphics[width=0.9\linewidth]{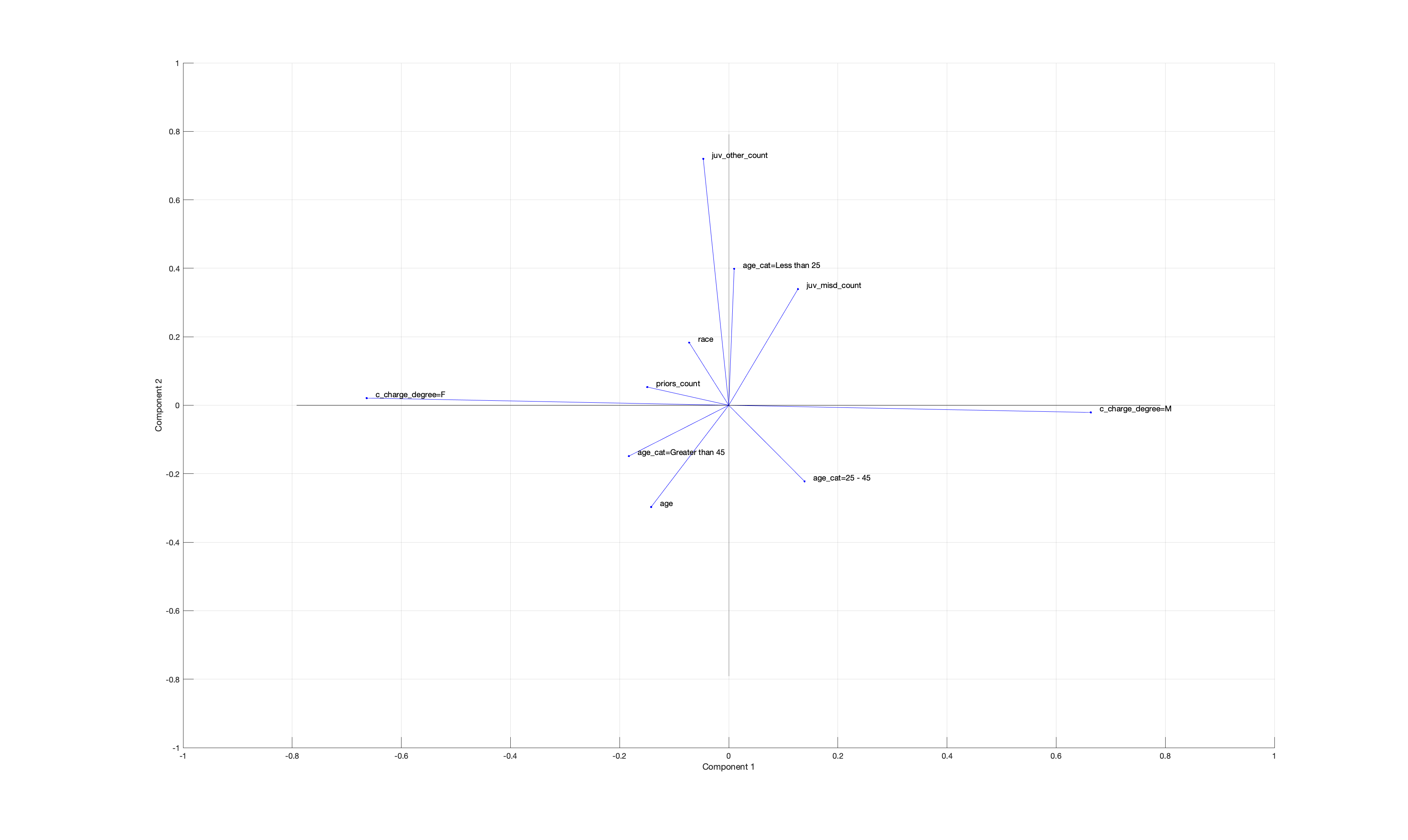}
    \caption{PC biplot of \textsc{FPCA} for COMPAS dataset}
    \label{fig:compas-biplot-fpca}
\end{figure*}
\begin{figure*}[!h]
    \centering
    \includegraphics[width=0.9\linewidth]{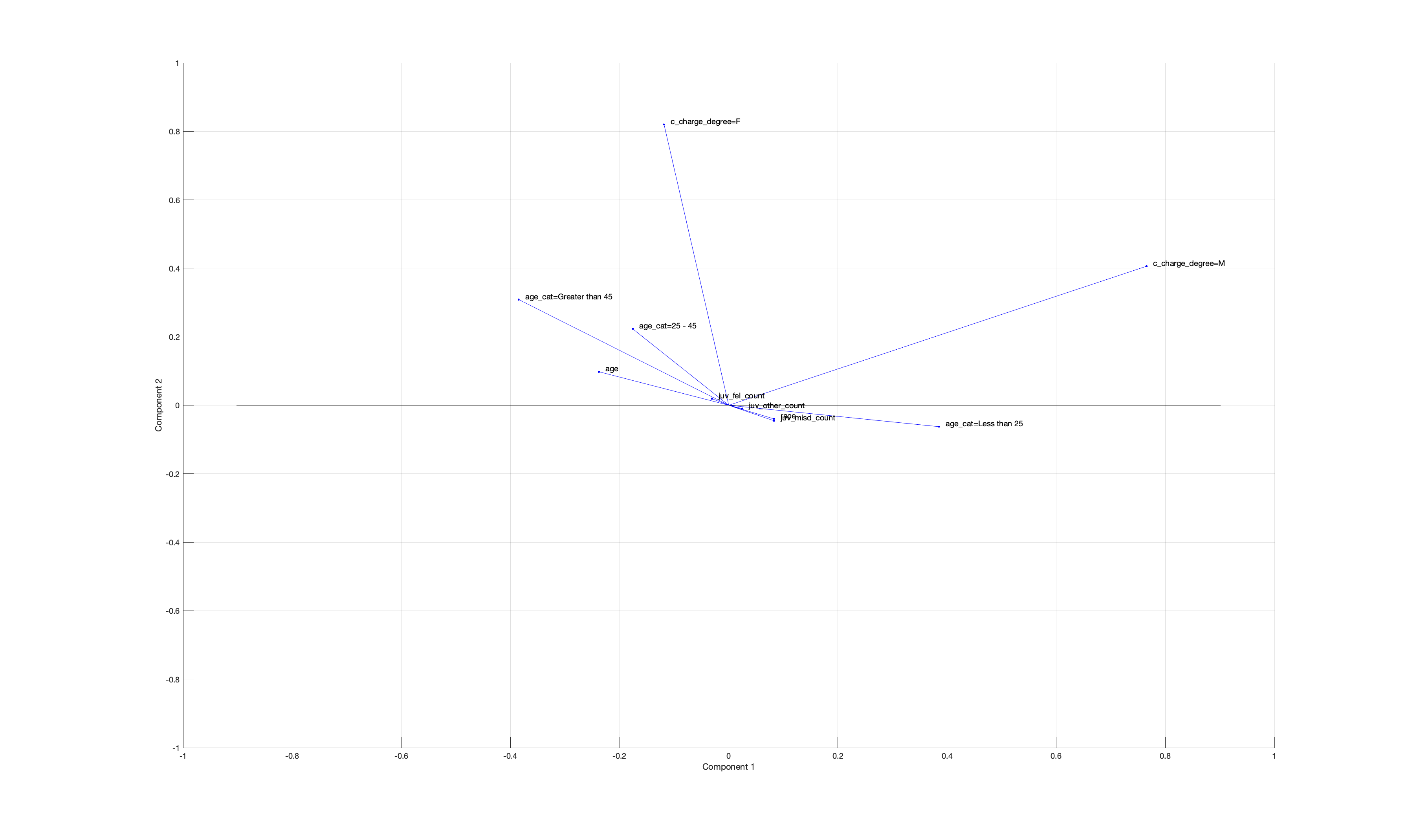}
    \caption{PC biplot of \textsc{MbF-PCA} for COMPAS dataset}
    \label{fig:compas-biplot-mbfpca}
\end{figure*}
\vfill

\begin{figure*}[!h]
	\begin{center}
		\begin{subfigure}[t]{0.49\linewidth}
				\includegraphics[width=\linewidth]{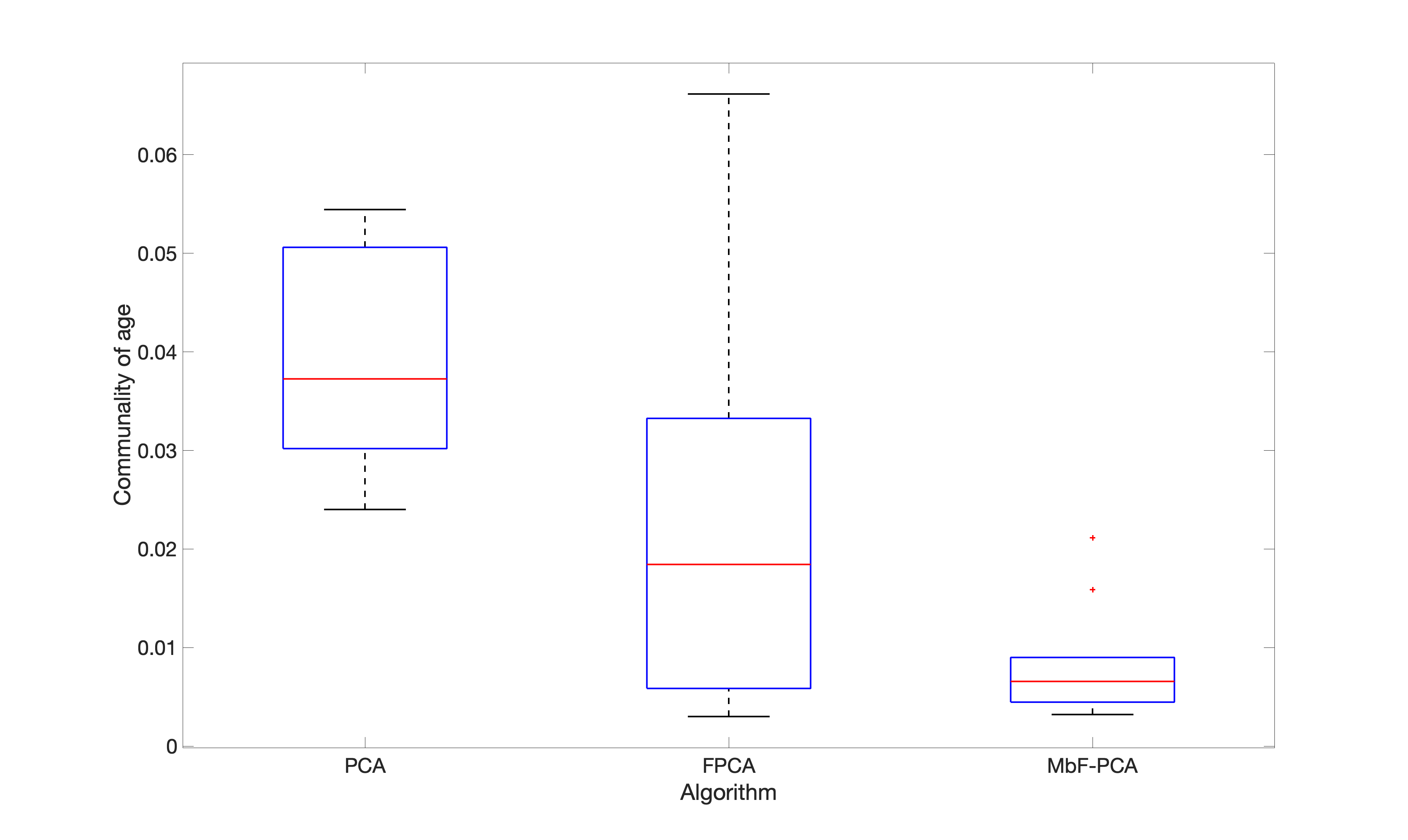}
				\caption{\label{fig:german-2-box} $d = 2$}
			\end{subfigure}\hfill
		\begin{subfigure}[t]{0.49\linewidth}
				\includegraphics[width=\linewidth]{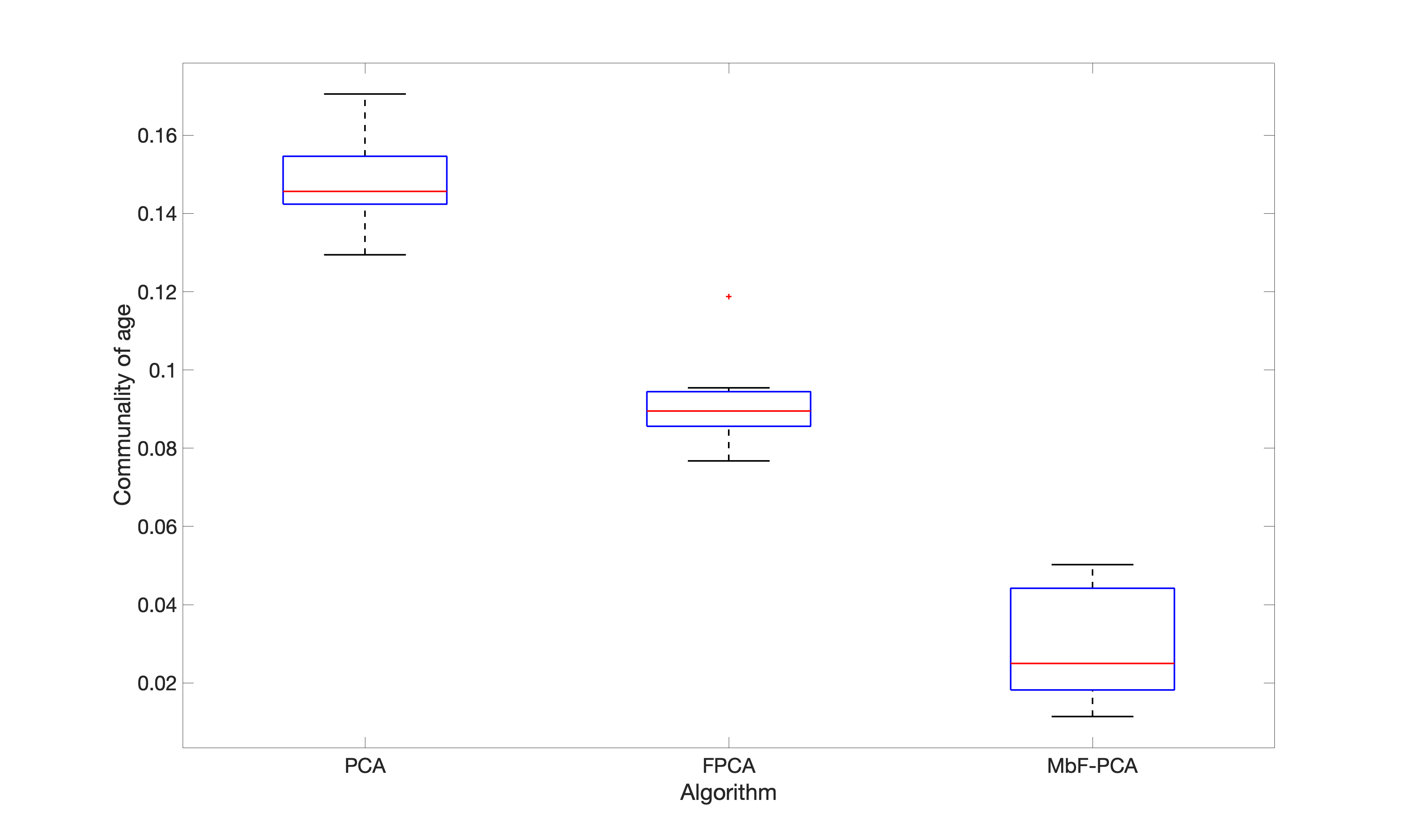}
				\caption{\label{fig:german-10-box} $d = 10$}
			\end{subfigure}
		\caption{\label{fig:german-box} Boxplots of communalities of age for German credit dataset}
	\end{center}
\end{figure*}
\begin{figure*}[!h]
    \centering
    \includegraphics[width=0.9\linewidth]{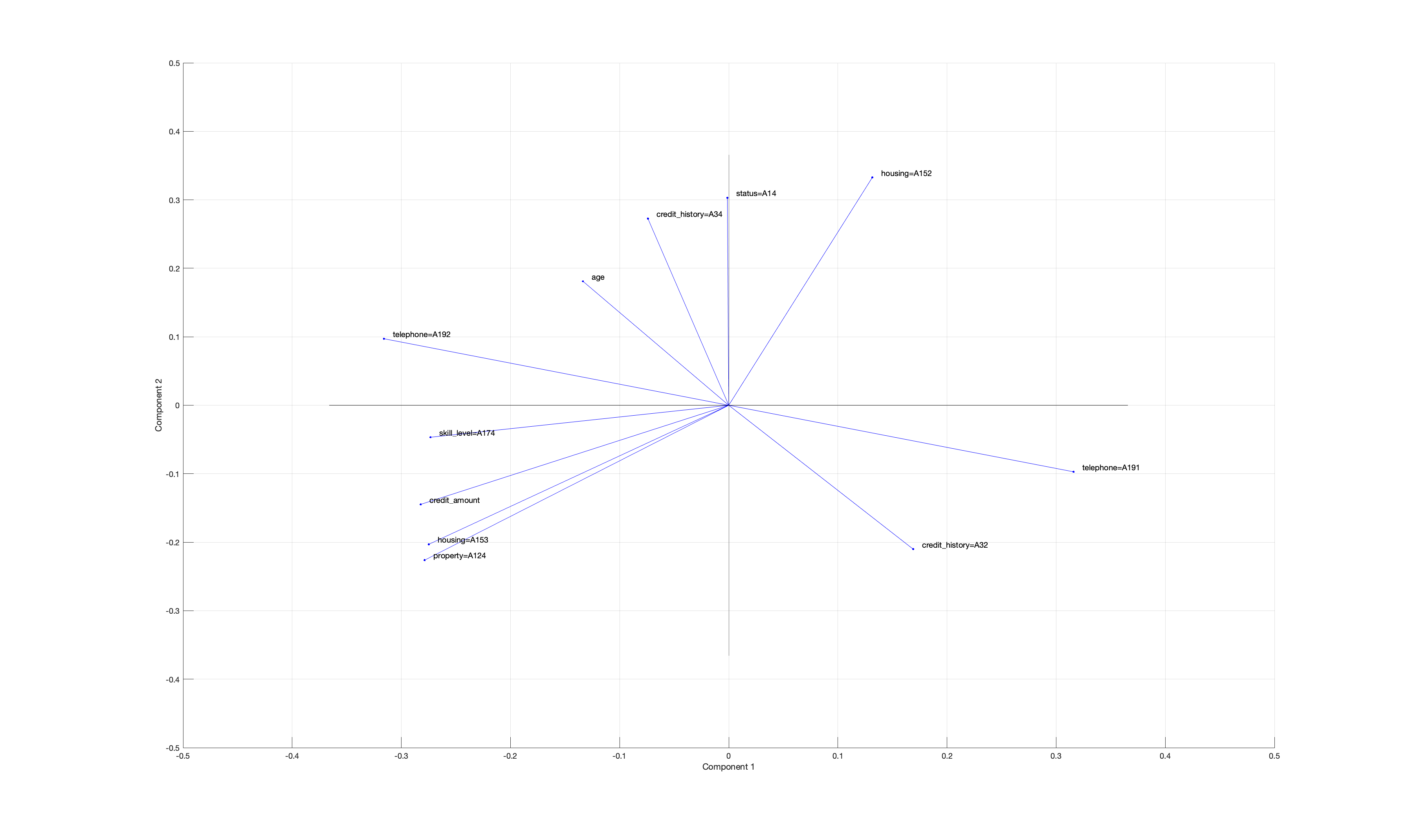}
    \caption{PC biplot of PCA for German credit dataset}
    \label{fig:german-biplot-pca}
\end{figure*}
\begin{figure*}[!h]
    \centering
    \includegraphics[width=0.9\linewidth]{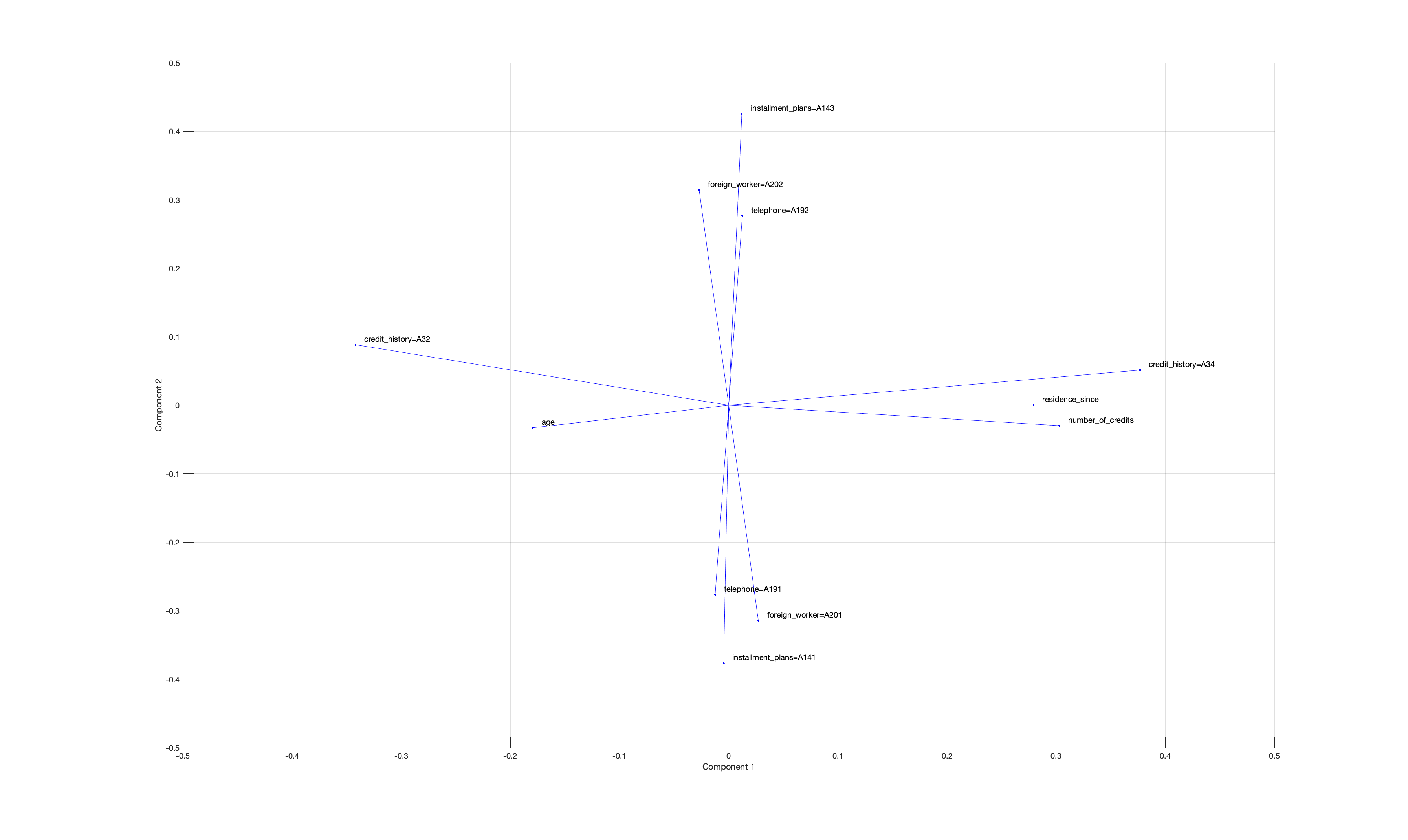}
    \caption{PC biplot of \textsc{FPCA} for German credit dataset}
    \label{fig:german-biplot-fpca}
\end{figure*}
\begin{figure*}[!h]
    \centering
    \includegraphics[width=0.9\linewidth]{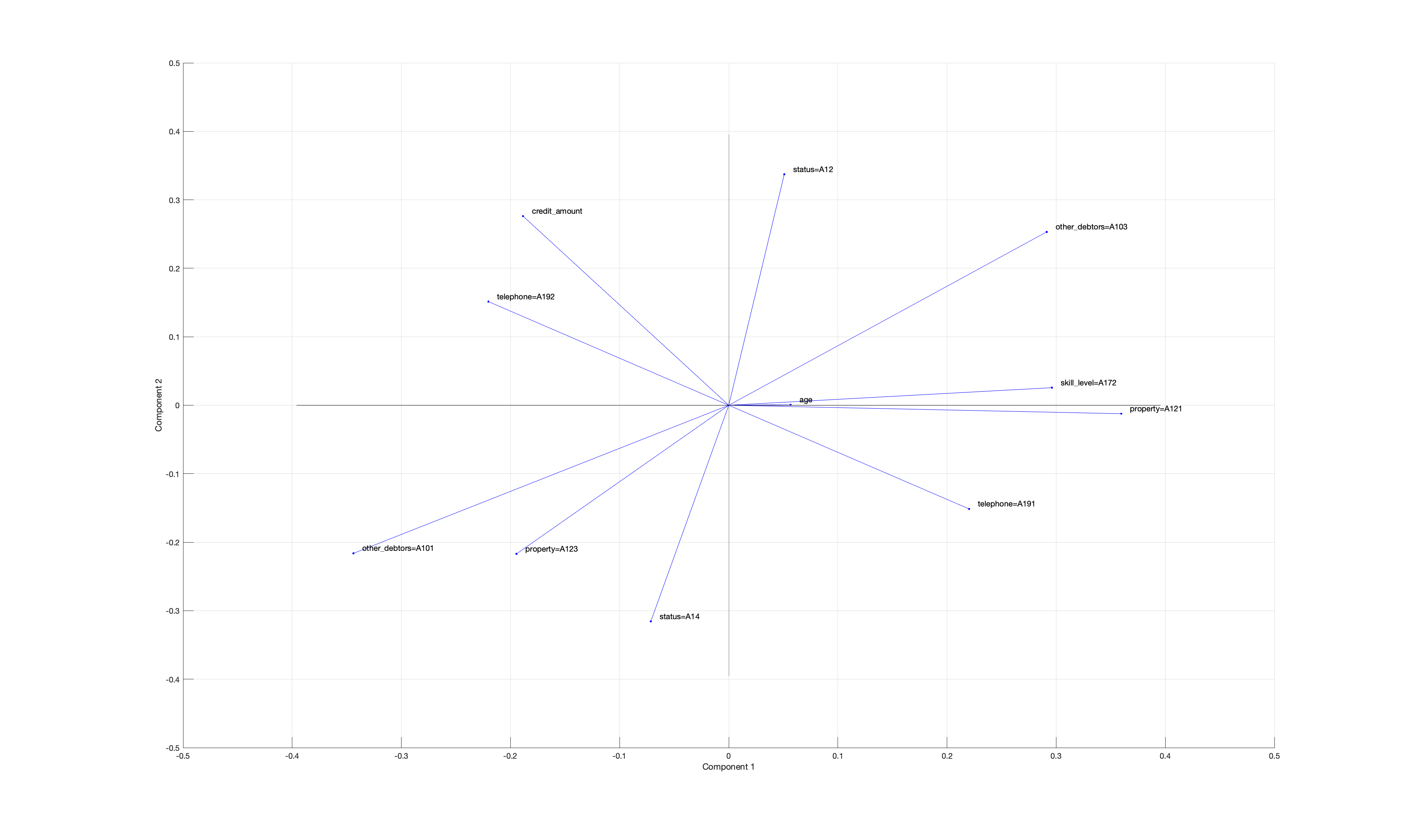}
    \caption{PC biplot of \textsc{MbF-PCA} for German credit dataset}
    \label{fig:german-biplot-mbfpca}
\end{figure*}
\vfill

\begin{figure*}[!h]
	\begin{center}
		\begin{subfigure}[t]{0.49\linewidth}
				\includegraphics[width=\linewidth]{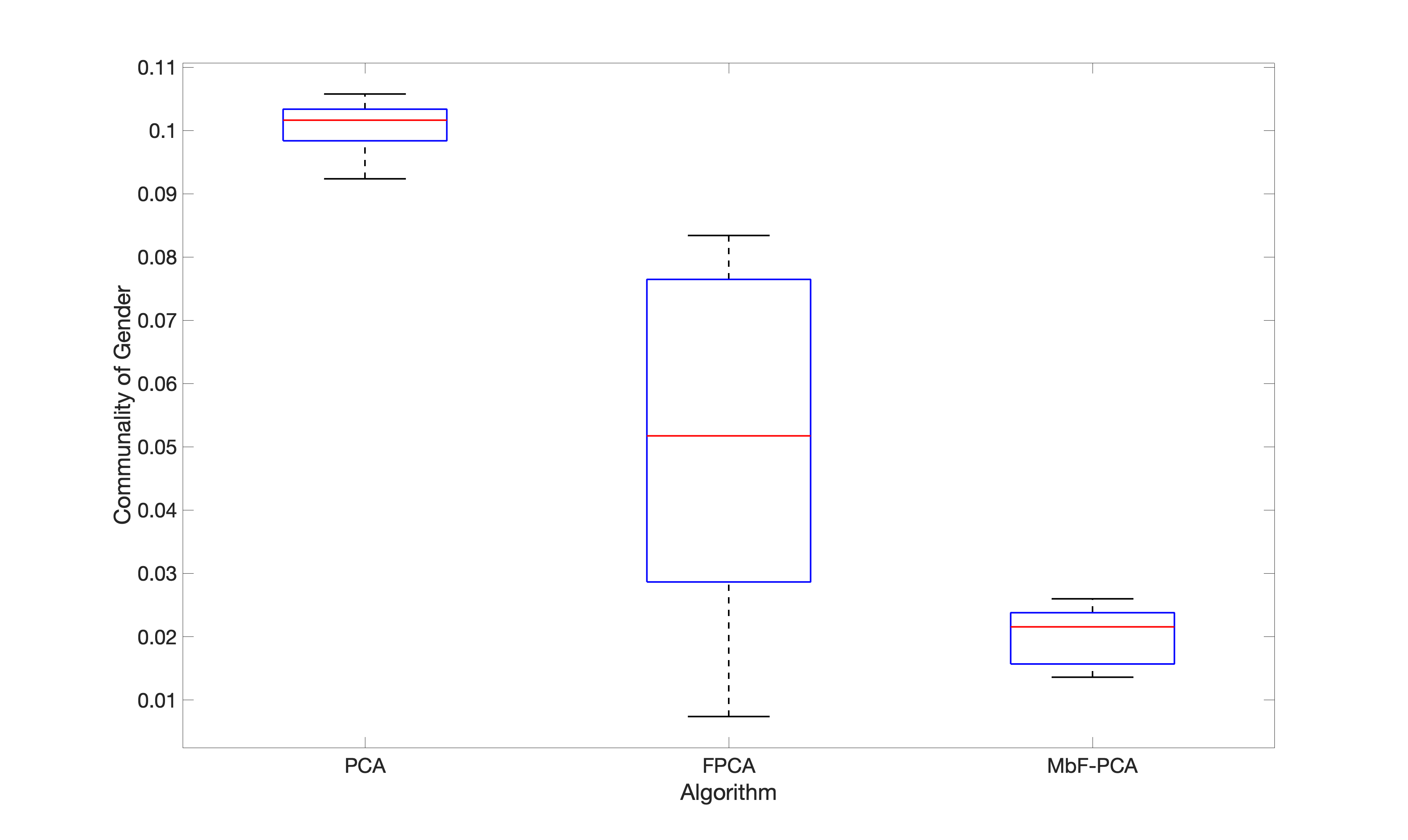}
				\caption{\label{fig:adult-2-box} $d = 2$}
			\end{subfigure}\hfill
		\begin{subfigure}[t]{0.49\linewidth}
				\includegraphics[width=\linewidth]{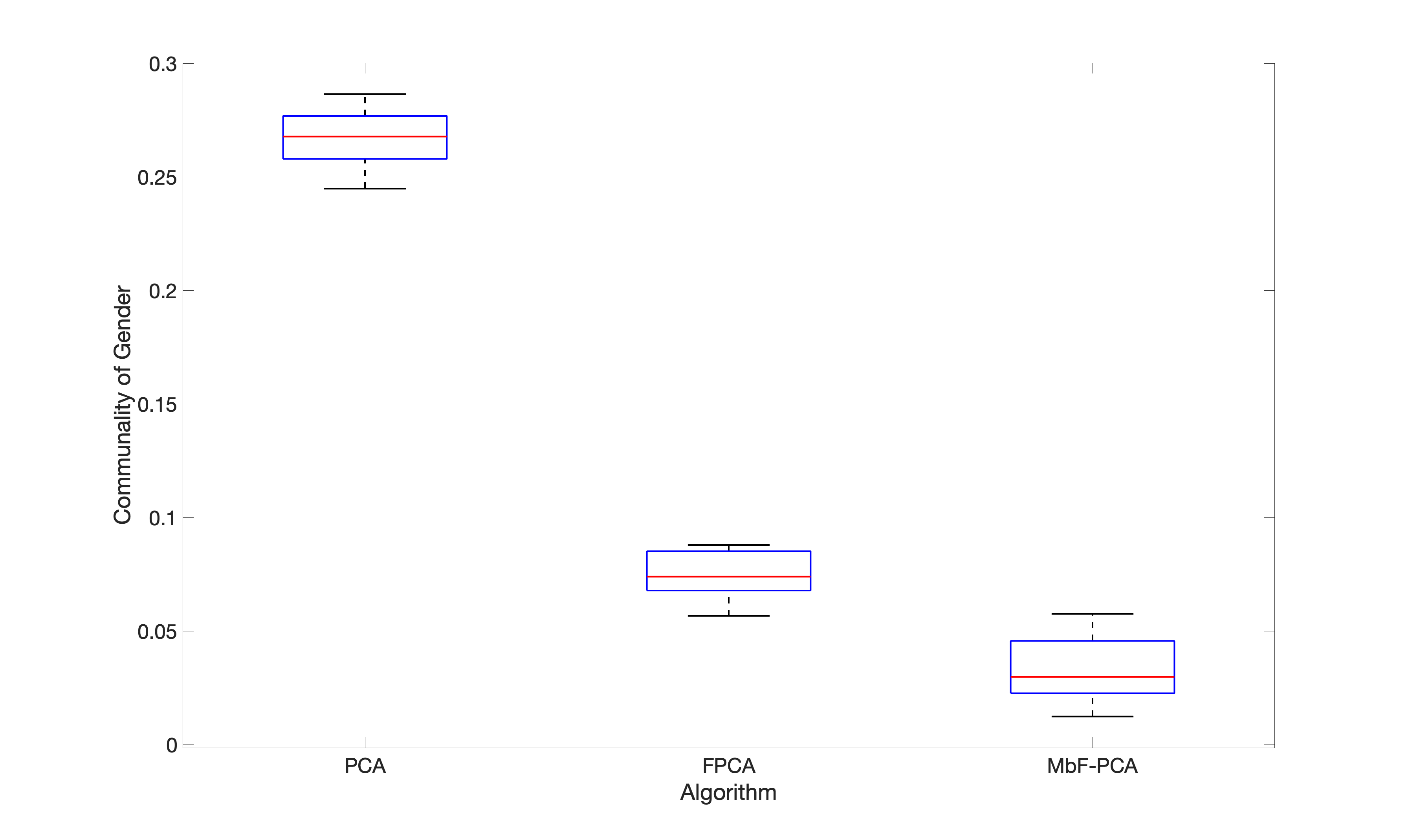}
				\caption{\label{fig:adult-10-box} $d = 10$}
			\end{subfigure}
		\caption{\label{fig:adult-box} Boxplots of communalities of gender(sex) for Adult income dataset}
	\end{center}
\end{figure*}
\begin{figure*}[!h]
    \centering
    \includegraphics[width=0.9\linewidth]{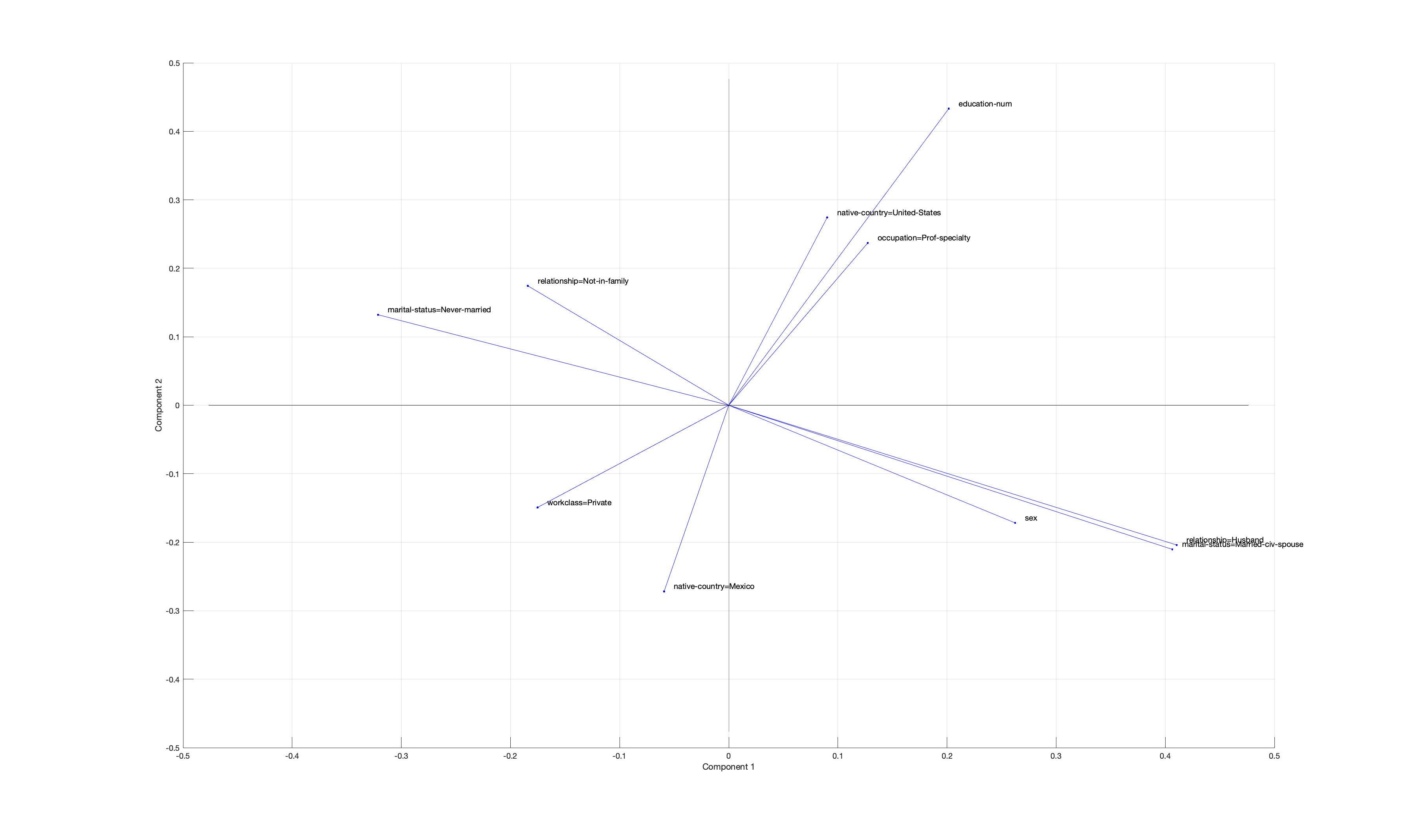}
    \caption{PC biplot of PCA for Adult income dataset}
    \label{fig:adult-biplot-pca}
\end{figure*}
\begin{figure*}[!h]
    \centering
    \includegraphics[width=0.9\linewidth]{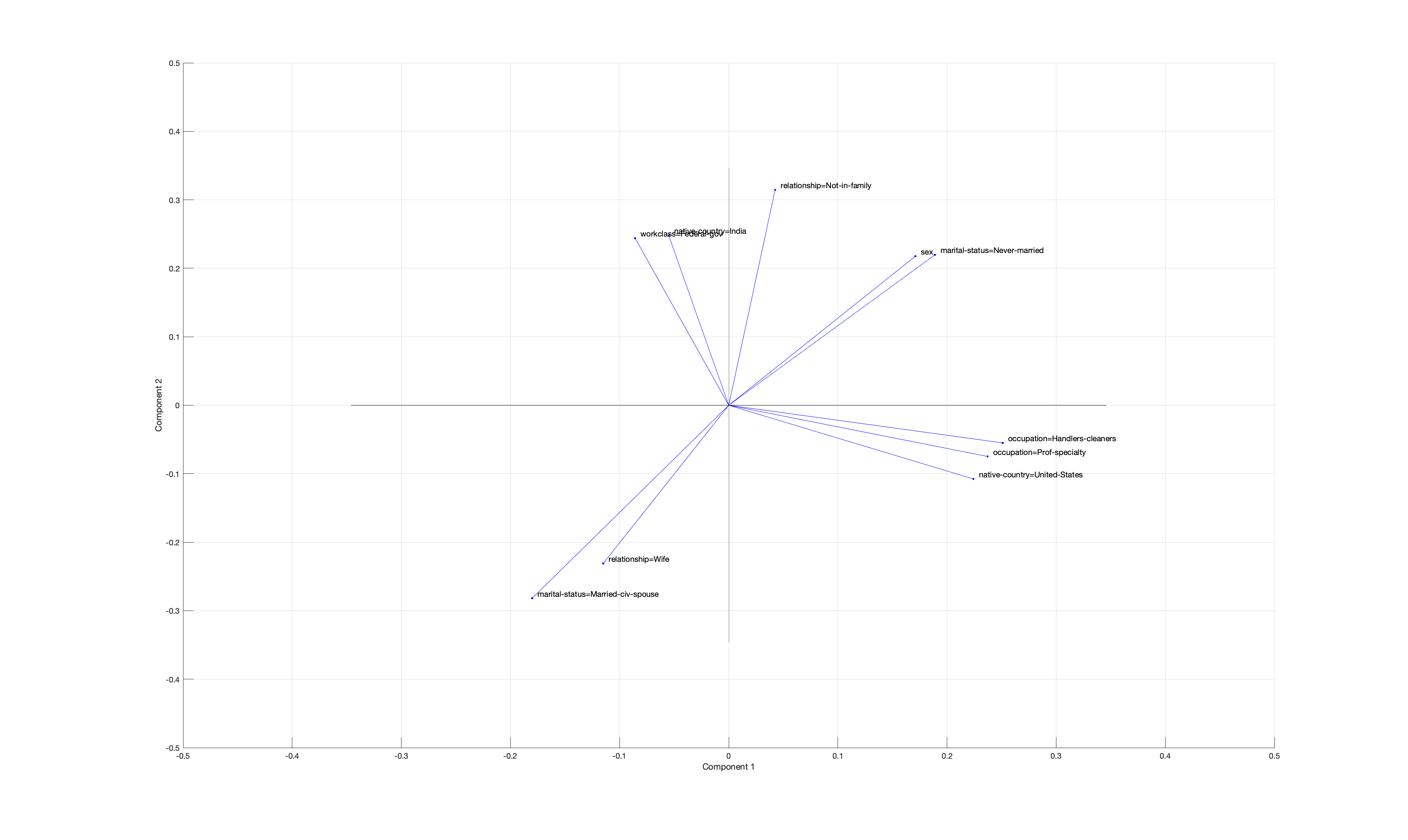}
    \caption{PC biplot of \textsc{FPCA} for Adult income dataset}
    \label{fig:adult-biplot-fpca}
\end{figure*}
\begin{figure*}[!h]
    \centering
    \includegraphics[width=0.9\linewidth]{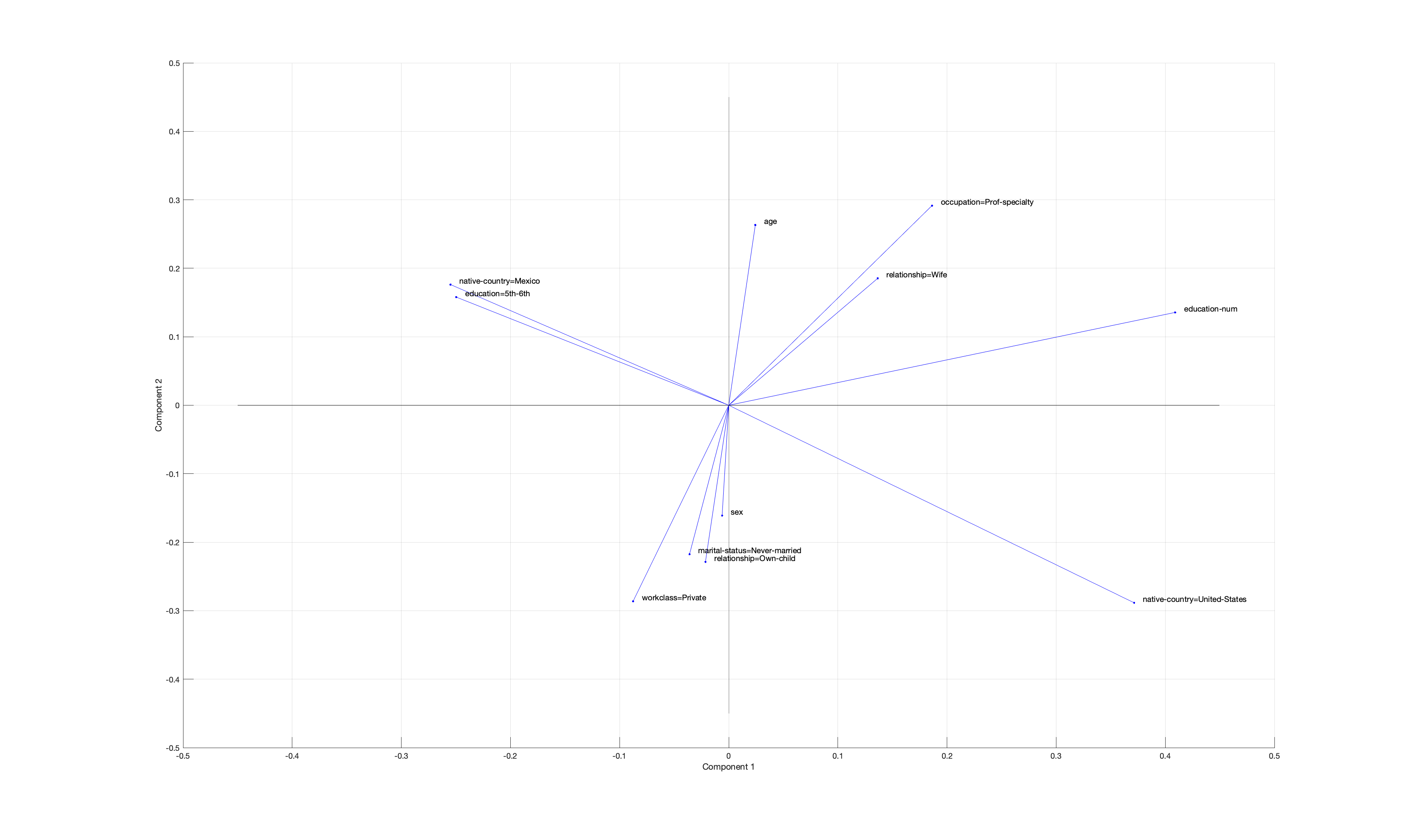}
    \caption{PC biplot of \textsc{MbF-PCA} for Adult income dataset}
    \label{fig:adult-biplot-mbfpca}
\end{figure*}
\vfill

\end{document}